\documentclass{article}
\usepackage{algorithmic}
\usepackage{graphicx}
\usepackage{tikz}
\usepackage{enumitem}
\usepackage{subcaption}
\usepackage{booktabs}
\usepackage{bm}
\usepackage{tcolorbox}
\usepackage{multicol,multirow}
\usepackage[sort]{natbib}
\usepackage{comment}
\usepackage{booktabs}
\newcommand{\ts}{\textstyle}
\newcommand{\ie}{\textit{i.e.}}
\newcommand{\eg}{\textit{e.g.}}

\usepackage{enumerate}   
\usepackage{amsmath}
\usepackage{tikz}
\usetikzlibrary{positioning,arrows}
\usepackage{wrapfig}

\DeclareMathOperator*{\argmin}{arg\,min}

\usepackage{amsmath}
\usepackage{appendix}
\usepackage{amssymb}

\usepackage{amsthm}
\usepackage{hyperref}
\usepackage[capitalise]{cleveref}
\Crefname{assumption}{Assumption}{Assumptions}

\theoremstyle{plain}
\newtheorem{theorem}{Theorem}
\newtheorem{lemma}[theorem]{Lemma}

\newtheorem{assumption}{Assumption}

\newtheorem{remark}{Remark}
\newtheorem{example}{Example}

\newcommand{\MDP}{\mathrm{MDP}}
\newcommand{\NMDP}{\mathrm{NMDP}}
\newcommand{\Proj}{\operatorname{Proj}}

\newcommand{\bG}{\mathbb{G}}
\def\Holder{{H\"{o}lder}}

\newcommand{\cm}{\mathcal{M}}
\newcommand{\ch}{\mathcal{H}}
\newcommand{\bigO}{\mathcal{O}} %

\newcommand{\traj}{\mathcal{T}}
\newcommand{\ck}{\traj}
\newcommand{\cu}{\mathcal{U}}

\newcommand{\op}{\mathrm{o}_{p}}
\newcommand{\E}{\mathbb{E}}
\newcommand{\Pa}{\mathrm{Pa}}
\newcommand{\oper}{\mathrm{op}}

\newcommand{\mI}{\mathcal{I}}

\newcommand{\var}{\mathrm{var}}
\newcommand{\cov}{\mathrm{cov}}

\newcommand{\thpol}{\pi^{\theta}}

\newcommand{\DO}{\mathrm{EOPPG}}

\newcommand{\prns}[1]{\left(#1\right)}
\newcommand{\braces}[1]{\left\{#1\right\}}
\newcommand{\bracks}[1]{\left[#1\right]}

\newcommand{\abs}[1]{\left|#1\right|}
\newcommand{\Rl}{\mathbb{R}}
\newcommand{\R}[1]{\mathbb{R}^{#1}}
\newcommand{\epol}{\pi^e}
\newcommand{\bpol}{\pi^b}

\def\Cramer{Cram\'{e}r}

\newcommand{\Fcal}{\mathcal{F}}

\newcommand{\RN}[1]{%
  \textup{\uppercase\expandafter{\romannumeral#1}}%
}

\usepackage{color,graphicx,lscape,longtable}

\def\boxit#1{\vbox{\hrule\hbox{\vrule\kern6pt\vbox{\kern6pt#1\kern6pt}\kern6pt\vrule}\hrule}}

\usepackage{xcolor}
\newcount\Comments  
\newcommand{\kibitz}[2]{\ifnum\Comments=1\textcolor{#1}{#2}\fi}

\def\che{\tikz\fill[scale=0.4](0,.35) -- (.25,0) -- (1,.7) -- (.25,.15) -- cycle;} 

\usepackage{microtype}
\usepackage{graphicx}
\usepackage{algorithm}
\usepackage{hyperref}

\usepackage[sort]{natbib}
\allowdisplaybreaks

\title{Statistically Efficient Off-Policy Policy Gradients}

\author{Nathan Kallus\\
       Cornell University\\
       New York, NY 10044, USA
       \and
        Masatoshi Uehara\thanks{\url{uehara_m@g.harvard.edu}}\\
       Harvard University\\
       Cambridge, MA 02138, USA}
\date{}

\begin{document}

\maketitle

\begin{abstract}
Policy gradient methods in reinforcement learning update policy parameters by taking steps in the direction of an estimated gradient of policy value. In this paper, we consider the statistically efficient estimation of policy gradients from off-policy data, where the estimation is particularly non-trivial. We derive the asymptotic lower bound on the feasible mean-squared error in both Markov and non-Markov decision processes and show that existing estimators fail to achieve it in general settings. We propose a meta-algorithm that achieves the lower bound without any parametric assumptions and exhibits a unique 3-way double robustness property. We discuss how to estimate nuisances that the algorithm relies on. Finally, we establish guarantees on the rate at which we approach a stationary point when we take steps in the direction of our new estimated policy gradient.
\end{abstract}

\section{Introduction}\label{sec:intro}

Learning sequential decision policies from observational off-policy data is an important problem in settings where exploration is limited and simulation is unreliable. A key application is reinforcement learning (RL) for healthcare \citep{gottesman2019guidelines}. In such settings, data is limited and it is crucial to use the available data \emph{efficiently}. Recent advances in off-policy evaluation \citep{KallusUehara2019,KallusNathan2019EBtC} have shown how efficiently leveraging problem structure, such as Markovianness, can significantly improve off-policy evaluation and tackle such sticky issues as the curse of horizon \citep{liu2018breaking}. An important next step is to translate these successes in off-policy \emph{evaluation} to off-policy \emph{learning}. In this paper we tackle this question by studying the efficient estimation of the \emph{policy gradient} from off-policy data and the implications of this for learning via estimated-policy-gradient ascent.

Policy gradient algorithms \citep[Chapter 13]{sutton2018reinforcement} enable one to effectively learn complex, flexible policies in potentially non-tabular, non-parametric settings and are therefore very popular in both on-policy and off-policy RL. We begin by describing the problem and our contributions, before reviewing the literature in \cref{sec:lit}.

\begin{table}[!]
    \centering
      \caption{Comparison of off-policy policy gradient estimators. Here, $f=\Theta(g)$ means $0<\liminf f/g\leq\limsup f/g<\infty$ (not to be confused with the policy parameter space $\Theta$). 
      In the second row, nuisances must be estimated at $n^{-1/2}$-rates, and in the rows below it, nuisances may be estimated at slow non-parametric rates.
      }
    {\small\begin{tabular}{lccc}\toprule
      Estimator        &  MSE & Efficient & Nuisances \\\midrule
     Reinforce, \cref{eq:step_reinforce}  &  $2^{\Theta(H)}\Theta(1/n)$  & & none \\
     PG, \cref{eq:pg_q}  & $2^{\Theta(H)}\Theta(1/n)$ & &  $q$ (parametric) \\
     EOPPG (NMDP)  & $2^{\Theta(H)}\Theta(1/n)$ & \che & $q,\nabla q$  \\
     EOPPG (MDP) &$\Theta(H^{4}/n)$ & \che  & $q,\mu,\nabla q,\nabla \mu$  \\\bottomrule
    \end{tabular}}
    \label{tab:comparison}
\end{table}

Consider a $(H+1)$-long Markov decision process (MDP), with states $s_t\in\mathcal S_t$, actions $a_t\in\mathcal A_t$, rewards $r_t\in\mathbb R$, initial state distribution $p_0(s_0)$, transition distributions $p_t(s_{t+1}\mid s_t,a_t)$, and reward distribution $p_t(r_{t}\mid s_t,a_t)$, for $t=0,\dots,H$. 
A policy $(\pi_t(a_t\mid s_t))_{t\leq H}$ induces a distribution over trajectories $\traj=(s_0,a_0,r_0,\dots,s_T,a_H,r_H,s_{H+1})$:
\begin{align}\ts \label{eq:trajdist_mdp}
&p_\pi(\traj)=\\\notag&~~~~~~\ts p_0(s_0)\prod_{t=0}^H\pi_t(a_t\mid s_t)p_t(r_t\mid s_t,a_t)p_t(s_{t+1}\mid s_t,a_t).
\end{align}
Given a class of policies $\pi^{\theta}_{t}(a_t\mid s_t)$ parametrized by $\theta\in\Theta \in \Rl^{D}$, we seek the parameters with greatest average reward, defined as
\begin{equation}\ts
J(\theta)=\E_{p_{\pi^\theta}}\bracks{\sum_{t=0}^Hr_t}.\tag{Policy Value}
\end{equation}
A generic approach is to repeatedly move $\theta$ in the direction of the policy gradient (PG), defined as
\begin{align}\ts
Z(\theta)&=\nabla_\theta J(\theta)
\tag{Policy Gradient}\\
&\ts=\E_{p_{\pi^\theta}}\bracks{\sum_{t=0}^Hr_t\sum_{k=0}^t\nabla_\theta \log\pi^{\theta}_{k}(a_k\mid s_k)}\notag
\end{align}
For example, in the \emph{on-policy} setting, 
we can generate trajectories from $\pi^\theta$, $\traj^{(1)},\dots,\traj^{(n)}\sim p_{\pi^\theta}$, and the (GPOMDP variant of the) REINFORCE algorithm \citep{baxter2001infinite} advances in the direction of the stochastic gradient
\begin{equation}
\hat Z^\text{on-policy}(\theta)=\frac1n\sum_{i=1}^n\sum_{t=0}^Hr_t^{(i)}\sum_{k=0}^t\nabla_\theta \log\pi^\theta_{k}(a_k^{(i)}\mid s_k^{(i)}).
\notag
\end{equation}

In the \emph{off-policy} setting, however, we \emph{cannot} generate trajectories from any given policy and, instead our data consists only of trajectory observations from one fixed policy,
\begin{equation}\tag{Off-policy data}\label{eq:offpolicydata}
\traj^{(1)},\dots,\traj^{(n)}\sim p_{\pi^b},
\end{equation}
where $\pi^b$ is known as the \emph{behavior policy}.
With this data, $\hat Z^\text{on-policy}(\theta)$ is no longer a stochastic gradient (\ie, it is biased \emph{and} inconsistent) and we must seek other ways to estimate $Z(\theta)$ in order to make policy gradient updates.

This paper addresses the \emph{efficient} estimation of $Z(\theta)$ from off-policy data and its use in off-policy policy learning. Specifically, our contributions are:
\begin{description}
\item [(\cref{sec:efficinecy})] We calculate the asymptotic lower bound on the minimal-feasible mean square error in estimating the policy gradient, which is of order $\bigO(H^{4}/n)$. In addition, we demonstrate that existing off-policy policy gradient approaches fail to achieve this bound and may even have exponential dependence on the horizon. 
\item [(\cref{sec:algorithm1})] We propose a meta-algorithm called Efficient Off-Policy Policy Gradient (EOPPG) that achieves this bound without any parametric assumptions. In addition, we prove it enjoys a unique 3-way double robustness property. 
\item [(\cref{sec:algorithm2})] We show how to estimate the nuisance functions needed for our meta-algorithm by introducing the concepts of Bellman equations for the gradient of $q$-function and stationary distributions. 
\item [(\cref{sec:optimization})] We establish guarantees for the rate at which we approach a stationary point when we take steps in the direction of our new estimated policy gradient. Based on efficiency results for our gradient estimator, we can prove the regret's horizon dependence is $H^2$.
\end{description}

\subsection{Notation and definitions}\label{sec:notationanddefs}

We define the following variables (note the implicit dependence on $\theta$):
\begin{align}
g_t&\ts=\nabla_\theta \log\pi_{\theta,t}(a_t\mid s_t),\tag{Policy score}\\
q_t&\ts=\E_{p_{\pi^\theta}}\bracks{\sum_{k=t}^Hr_t\mid s_t,a_t},
\tag{$q$-function}\\
v_t&\ts=\E_{p_{\pi^\theta}}\bracks{\sum_{k=t}^Hr_t\mid s_t},\tag{$v$-function}\\
\tilde\nu_t&\ts=\frac{\pi_{t}^\theta(a_t\mid s_t)}{\pi_{t}^b(a_t\mid s_t)},\tag{Density Ratio}\\
\nu_{t':t}&\ts=\prod_{k=t'}^t
\tilde\nu_k
,
\tag{Cumulative Density Ratio}\\
\tilde\mu_t&\ts=\frac{p_{\pi^\theta}(s_t)}{p_{\pi^b}(s_t)}
\tag{Marginal State Density Ratio}\\
\mu_{t}&\ts=\tilde\mu_t\tilde\nu_t
,\tag{Marginal State-Action Density Ratio}\\
d^{q}_t&=\nabla_\theta q_t,d^{v}_t=\nabla_\theta v_t,d^{\mu}_t=\nabla_\theta \mu_t,d^{\nu}_t=\nabla_\theta \nu_{0:t}.\notag
\end{align}
Note that all of the above are simply functions of the trajectory, $\traj$, and $\theta$. To make this explicit, we sometimes write, for example, $q_t=q_t(s_t,a_t)$ and refer to $q_t$ as a function. Similarly, when we estimate this function by $\hat q_t$ we also refer to $\hat q_t$ as the random variable gotten by evaluating it on the trajectory, $\hat q_t(s_t,a_t)$.

We write $a \lnapprox b $ to mean that there exists a \emph{universal} constant $C$ such that $a\leq Cb$. 
We let  $\|\cdot\|_{2}$ denote the Euclidean vector norm and $\|\cdot\|_{\oper}$ denote the matrix operator norm.

All expectations, variances, and probabilities \emph{without} subscripts are understood to be with respect to $p_{\pi^b}$.
Given a vector-valued function of trajectory, $f$, we define
its $L_2$ norm as $$
\|f\|_{L^2_b}^2=\E\|f(\traj)\|_2^2.$$
Further, given trajectory data, $\traj^{(1)},\dots,\traj^{(n)}$,
we define the \emph{empirical expectation} as
$$
\ts\E_nf=\E_n[f(\traj)]=\frac1n\sum_{i=1}^nf(\traj^{(i)}).
$$

\paragraph{MDP and NMDP.} Throughout this paper, we focus on the MDP setting where the trajectory distribution $p_\pi$ is given by \cref{eq:trajdist_mdp}. For completeness, we also consider relaxing the Markov assumption, yielding a \emph{non-Markov} decision process (NMDP), where the trajectory distribution $p_{\pi}(\traj)$ is
\begin{align}\ts \notag
p_0(s_0)\prod_{t=0}^H\pi_t(a_t\mid \ch_{s_t})p_t(r_t\mid \ch_{a_t})p_t(s_{t+1}\mid \ch_{a_t}), 
\end{align}
where $\ch_{a_t}$ is $(s_0,a_0,\cdots,a_t)$ and $\ch_{s_t}$ is $(s_0,a_0,\cdots,s_t)$.

\paragraph{Assumptions.} Throughout we assume that $\forall t\leq H$: $0\leq r_t\leq R_{\max},\,\|g_t\|_{\oper}\leq G_{\max},\,\tilde\nu_{t} \leq C_1,\,\tilde\mu_t \leq C'_2$. And, we define $C_2=C_1C'_2$ so that $\mu_t\leq C_2$.

\subsection{Related literature}\label{sec:lit}

\subsubsection{Off-policy policy gradients}

A standard approach to dealing with off-policy data is to correct the policy gradient equation using \emph{importance sampling} (IS) using the cumulative density ratios, $\nu_{0:t}$ (see, \eg, \citealp[Appendix A]{papini2018stochastic}; \citealp{AAAISSS2018-Hanna}). 
This allows us to rewrite the policy gradient $Z(\theta)$ as an expectation over $p_{\pi^b}$ and then estimate it using an equivalent empirical expectation.

The off-policy version of the classic REINFORCE algorithm \citep{williams1992simple} recognizes
\begin{align}
\ts
\label{eq:reinforce}
   Z(\theta)=\E\bracks{\nu_{0:H}\prns{\sum_{t=0}^Hr_t}\prns{\sum_{t=0}^Hg_t}} 
\end{align}
(recall that $\E$ is understood as $\E_{p_{\pi^b}}$)
and uses the estimated policy gradient given by replacing $\E$ with $\E_n$. (Similarly, if $\nu_{0:H}$ is unknown it can be estimated and plugged-in.)
The GPOMDP variant \citep{baxter2001infinite} refines this by
\begin{equation}\label{eq:reinforce_gomdp}\ts
Z(\theta)=\E\bracks{\nu_{0:H}\sum_{t=0}^Hr_t\sum_{s=0}^t g_s},
\end{equation}
whose empirical version ($\E_n$) has \emph{less} variance and is therefore preferred. A further refinement is given by a step-wise IS \citep{precup2000eligibility} as in \citet{DeisenrothMarc2013ASoP}:
\begin{align}
\label{eq:step_reinforce}
\ts
Z(\theta)=\E\bracks{\sum_{t=0}^H\nu_{0:t}r_t\sum_{s=0}^t g_s}. 
\end{align}

Following \citet{DegrisThomas2013OA}, \citet{ChenMinmin2019TOCf} replace $\nu_{0:t}$ with $\tilde\nu_{t}$ in \cref{eq:step_reinforce} to reduce variance, but this is an \emph{approximation} that incurs non-vanishing bias.

By exchanging the order of summation in \cref{eq:step_reinforce}
and recognizing $q_t=\E\bracks{\sum_{j=t}^{H}\nu_{t+1:j}r_j\mid s_t,a_t}$,
we obtain a policy gradient in terms of the $q$-function \citep{sutton1998intra},
\begin{align}
\label{eq:pg_q}
\ts
      Z(\theta)&\ts=
      \E\left[\sum_{t=0}^{H}\nu_{0:t}g_t q_t\right].     
\end{align}

The off-policy policy gradient (Off-PAC) of \citet{DegrisThomas2013OA} is obtained by replacing $\nu_{0:t}$ with $\tilde\nu_{t}$ in \cref{eq:pg_q}, estimating $q_t$ by $\hat q_t$ and plugging it in, and taking the empirical expectation.
Replacing $\nu_{0:t}$ with $\tilde\nu_{t}$ is intended to reduce variance but it is an \emph{approximation} that ignores the state distribution mismatch (essentially, $\mu_t$) and incurs non-vanishing bias. 
Since it amounts to a reweighting and the unconstrained optimal policy remains optimal on any input distribution, 
in the tabular and fully unconstrained case considered in \citet{DegrisThomas2013OA}, we may still converge. But this fails in the general non-parametric, non-tabular setting. We therefore focus only on \emph{consistent} estimates of $Z(\theta)$ in this paper (which requires zero or vanishing bias).

Many of the existing off-policy RL algorithms including DDPG \citep{silver14} and Off-PAC with emphatic weightings \citep{Imani2018} also use the above trick, \ie, ignoring the state distribution mismatch. Various recent work deals with this problem \citep{LiuYao2019OPGw,TosattoSamuele2020ANOP,DaiBo2019APGf}. These, however, both assume the existence of a stationary distribution and are not efficient. We do not assume the existence of a stationary distribution since many RL problems have a finite horizon and/or do not have a stationary distribution. Moreover, our gradient estimates are efficient in that they achieve the MSE lower bound among all regular estimators.

\subsubsection{Other literatue}
\paragraph{Online off-policy PG.}
Online policy gradients have shown marked success in the last few years \citep{Schulmanetal_NIPS2015}. Various work has investigated incorporating offline information into online policy gradients \citep{gu2017,metelli2018}. Compared with this setting, our setting is completely off-policy with no opportunity of collecting new data from arbitrary policies, as considered in, \eg, \citet{kallus2018balanced,kallus2018confounding,swaminathan2015counterfactual,AtheySusan2017EPL} for $H=0$ and \citet{ChenMinmin2019TOCf,fujimoto19a} for $H\geq1$. 

\paragraph{Variance reduction in PG.} Variance reduction has been a central topic for PG \citep{jie2010,Greensmit2004,Schulmanetal_ICLR2016,WuCathy2018VRfP}. These papers generally consider a given class of estimators given by an explicit formula (such as given by all possible baselines) and show that some estimator is optimal among the class. In our work, the class of estimators among which we are optimal is \emph{all} regular estimators, which both extremely general and also provides minimax bounds in any vanishing neighborhood of $p_{\pi^b}$ \citep[Thm. 25.21]{VaartA.W.vander1998As}.

\paragraph{Off-policy evaluation (OPE).} 
OPE is the problem of estimating $J(\theta)$ for a given $\theta$ from off-policy data.
Step-wise IS \citep{precup2000eligibility} and direct estimation of $q$-functions \citep{munos2008finite} are two common approaches for OPE. However, the former is known to suffer from the high variance and the latter from model misspecification. To alleviate this, the doubly robust estimate combines the two; however, the asymptotic MSE still explodes exponentially in the horizon like $\Omega(C^{H}_1 H^2/n)$ \citep{jiang,thomas2016}. 
\citet{KallusUehara2019} show that the efficiency bound in the MDP case is actually polynomial in $H$ and give an estimator achieving it, which combines marginalized IS \citep{XieTengyang2019OOEf} and $q$-modeling using cross-fold estimation. This achieves MSE $O(C_2H^2/n)$. 
\citet{KallusNathan2019EBtC} further study efficient OPE in the 
infinite horizon MDP setting with non-iid batch data.

\section{Efficiency Bound for Estimating $Z(\theta)$}\label{sec:efficinecy}

Our target estimand is $Z(\theta)$ so a natural question is what is the least-possible error we can achieve in estimating it.
In parametric models, the \Cramer-Rao bound lower bounds the variance of all unbiased estimators and, due to \citet{hajek1970characterization}, also the asymptotic MSE of \emph{all} (regular) estimators. Our model, however, is nonparametric as it consists of \emph{all} MDP distributions, \ie, \emph{any} choice for $p_0(s_0)$, $p_t(r_t\mid s_t,a_t)$, $p_t(r_t\mid s_t,a_t)$, and $\pi_t(a_t\mid s_t)$ in \cref{eq:trajdist_mdp}.
Semiparametric theory gives an answer to this question. 
We first informally state the key efficiency bound result from semiparametric theory in terms of our own model, which is all MDP distributions, and our estimand, which is $Z(\theta)$.
For additional detail, see \cref{sec:semipara}; \citet{VaartA.W.vander1998As,TsiatisAnastasiosA2006STaM}. 

\begin{theorem}[Informal description of \citet{VaartA.W.vander1998As}, Theorem 25.20]\label{eq:vandervaartthm}
There exists a function $\xi_{\MDP}(\traj;p_{\pi^b})$ such that
for {any} MDP distribution $p_{\pi^b}$ and any regular estimator $\hat Z(\theta)$,
\begin{align*}
\ts\inf_{\|v\|_2\leq 1}
v^T(\operatorname{AMSE}[\hat Z(\theta)]-\var[\xi_{\MDP}])v\geq0,
\end{align*}
where $\operatorname{AMSE}[\hat Z(\theta)]=\int zz^TdF(z)$ is the second moment of $F$ the limiting distribution of $\sqrt{n}(\hat Z(\theta)-Z(\theta))$.%
\footnote{Note that by Fatou's lemma, we have that $\liminf_{n\to\infty}n\E[(v^T(\hat Z(\theta)-Z(\theta)))^2]\geq v^T\operatorname{AMSE}[\hat Z(\theta)]v$.}
\end{theorem}
$\xi_{\MDP}$ is called the \emph{efficient influence function} (EIF).
This also implies 
$\|{\operatorname{AMSE}[\hat Z(\theta)]}\|_{\oper}\geq \|{\var[\xi_{\MDP}]}\|_{\oper}$. Here, $\var[\xi_{\MDP}]$ is called the \emph{efficiency bound} (note it is a covariance \emph{matrix}). A regular estimator is any whose limiting distribution is insensitive to small changes of order $O(1/\sqrt{n})$ to $p_{\pi^b}$ that keep it an MDP distribution \citep[see][Chapter 8.5]{VaartA.W.vander1998As}. So the above provides a lower bound on the variance of {all} regular estimators, which is a very general class. It is so general that the bound also applies to \emph{all} estimators at all in a local asymptotic minimax sense \citep[see][Theorem 25.21]{VaartA.W.vander1998As}. 

Technically, we actually needed to prove that $\xi_{\MDP}$ exists in \cref{eq:vandervaartthm}. The following result does so and derives it explicitly. The one after does the same in the {NMDP} model. (Note that, while usually the EIF refers to a function with 0 mean, instead we let the EIF have mean $Z(\theta)$ everywhere as it simplifies the presentation. Since adding a constant does not change the variance, \cref{eq:vandervaartthm} is unchanged.)
\begin{theorem}
\label{thm:mdp}
The EIF of $Z(\theta)$ under MDP, $\xi_{\MDP}$, exists and is equal to
\begin{align*}
\ts
      \sum_{j=0}^{H} (d^{\mu}_j (r_j-q_j)-\mu_jd^{q}_j
      +\mu_{j-1}d^{v}_j +d^{\mu}_{j-1}v_j),
\end{align*}
where $\mu_{-1}=1,\,d^{\mu}_{-1}=0$. 

And, in particular,
\begin{align*}\ts
    \|\var[\xi_{\MDP}]\|_{\oper}\leq C_2R^2_{\max}G_{\max}(H+1)^2(H+2)^2/4.  
\end{align*}
\end{theorem}

\begin{theorem}
\label{thm:nmdp}
The EIF of $Z(\theta)$ under NMDP, $\xi_{\MDP}$, exists and is equal to
\begin{align*}\ts
\sum_{j=0}^{H}(d^{\nu}_j (r_j -q_j)-\nu_{0:j}d^{q}_j
      +\nu_{0:j-1}d^{v}_j +d^{\nu}_{j-1}v_j).
\end{align*}
where $\nu_{0:-1}=1,\,d^{\nu}_{-1}=0$. 
(Note that here $\nu_{0:j},d^{\nu}_j,d^q_j,d^q_j$ are actually functions of all of $\ch_{a_j}$ and not just of $(s_j,a_j)$ as in MDP case in \cref{thm:mdp}.)

And, in particular,
\begin{align*}\ts
\|\var[\xi_{\NMDP}]\|_{\oper} \leq C^{H}_1R^2_{\max}G_{\max}(H+1)^2(H+2)^2/4.
\end{align*}
\end{theorem}
Formulae for $\var[\xi_{\MDP}]$ and $\var[\xi_{\NMDP}]$ are given in \Cref{sec:proof}.
\cref{thm:nmdp} showed $\var[\xi_{\NMDP}]$ is at most exponential; we next show it is also at least exponential.
\begin{theorem}\label{thm:lower_bound_nmdp}Suppose that $\tilde\nu_{t}\geq C_3$
and that
$\var [(\sum_h g_h)(r_{H}-q_{H})\mid\ch_{a_{H}}]\succeq cI$.
Then, $\|\var[\xi_{\NMDP}]\|_{\oper}\geq C^{2H}_3c$. 
\end{theorem}
\Cref{thm:lower_bound_nmdp,thm:nmdp} show that the curse of horizon is generally \emph{unavoidable} in NMDP since the lower bound in is at least \emph{exponential} in horizon. On the other hand, \cref{thm:mdp} shows there is a possibility we can avoid the curse of horizon in MDP in the sense that the lower bound is at most polynomial in horizon; we show we can achieve this bound in \cref{sec:EOPPG}.

First, we show that REINFORCE (which is regular under NMDP) necessarily suffers from the curse of horizon.
\begin{theorem}
\label{thm:step_wise}
The MSE of step-wise REINFORCE \cref{eq:step_reinforce} is 
\begin{align*}\ts
    \sum_{k=0}^{H+1}\E[\nu^2_{k-1}\var[\E[\sum_{t=k-1}^H\nu_{k:t}r_t\sum_{s=k-1}^{t} g_s  \mid \ch_{a_k}]\mid \ch_{a_{k-1}}]],
\end{align*}
which is no smaller than the MSE of REINFORCE \cref{eq:reinforce} and GOMDP-REINFROCE \cref{eq:reinforce_gomdp}. (Equation references refer to the estimate given by replacing $\E$ with $\E_n$.)
\end{theorem} 
\begin{theorem}Suppose that $\tilde\nu_{t}\geq C_3$ and that
$
\var[r_H g_H \mid \ch_{a_{H}}]\succeq cI 
$.
Then, the operator norm of the variance of step-wise REINFORCE is lower bounded by $cC^{2H}_3/n$. 
\label{thm:step_wise_lower}
\end{theorem}

\section{Efficient Policy Gradient Estimation}
\label{sec:EOPPG}

In this section we develop an estimator, EOPPG, for $Z(\theta)$ achieving the lower bound in \cref{thm:mdp} under weak nonparametric rate assumptions. 

\subsection{The Meta-Algorithm}\label{sec:algorithm1}

Having derived the EIF of $Z(\theta)$ under MDP in \cref{thm:mdp}, we use a meta-algorithm based on estimating the unknown functions (aka nuisances) $\mu_j, d^{q}_j, q_j, d^{\mu}_j$ and plugging them into $\xi_\MDP$, as described in \cref{alg:EOPPG}.
In particular, we use a cross-fitting technique \citep{VaartA.W.vander1998As,ChernozhukovVictor2018Dmlf}.
We refer to this as a meta-algorithm as it relies on given nuisances estimators: we show to construct these in \cref{sec:nuisance}.
\begin{algorithm}[t!]
 \caption{Efficient Off-Policy Policy Gradient}
 \begin{algorithmic}
 \label{alg:EOPPG}
\STATE Take a $K$-fold random partition $(I_k)^K_{k=1}$ of the observation indices $\{1,\dots,n\}$ such that the size of each fold, $\abs{I_k}$, is within $1$ of $n/K$.
\STATE Let $\mathcal{L}_k=\{\traj^{(i)}:i\in I_k\},\,\mathcal{U}_k=\{\traj^{(i)}:i\notin I_k\}$
 \FOR{$k\in\{1,\cdots,K\}$}
    \STATE Using only $\mathcal{L}_k$ as data, construct nuisance estimators ${\hat q^{(k)}}_j,\,{\hat \mu^{(k)}}_j,\,\hat d^{q(k)}_j,\hat d^{\mu(k)}_j$ for $\forall j\leq H$ (see \cref{sec:algorithm2})
    \STATE By integrating/summing w.r.t $a_j\sim\thpol_j(a_j\mid s_j)$, set
    \begin{equation}\ts\hspace{-2em}
    \hat v_{j}(s_j)   = \E_{\thpol_j}[\hat q_{j} \mid s_j], ~
    \hat d^{v}_j(s_j) = \E_{\thpol_j}[\hat d^q_j+\hat q_jg_j \mid s_j]
    \label{eq:vhat}
    \end{equation}

    \STATE Construct an intermediate estimate:
    \begin{align*}\ts
    \hat{Z}_k(\theta)=\E_{\cu_k}\bigl[\ts\sum_{j=0}^{H}\bigl(&\hat d^{\mu(k)}_j (r_j- \hat q^{(k)}_j)-\hat \mu^{(k)}_j\hat d^{q(k)}_j \\ 
    &\ts+\hat \mu^{(k)}_{j-1}\hat d^{v(k)}_j +\hat d^{\mu(k)}_{j-1}\hat v^{(k)}_j\bigr)\bigr],
\end{align*}
where $\E_{\cu_k}$ is the empirical expectation over $\cu_k$
    \ENDFOR 
    \STATE Return
    $
    \hat Z^{\DO}(\theta) = \frac{1}{K}\sum^K_{k=1}\hat{Z}_k.
    $
\end{algorithmic}
\end{algorithm}

Note \cref{eq:vhat} in \cref{alg:EOPPG} is computed simply by taking an integral over $a_j$ (or, sum, for finite actions) with respect to the \emph{known} measure (or, mass function) $\thpol_j(a_j\mid s_j)$.

We next prove that EOPPG achieves the efficiency bound under MDP and enjoys a 3-way double robustness (see \cref{fig:weak}).
We require the following about our nuisance estimators, which arises from the boundedness assumed in \cref{sec:notationanddefs}. 
\begin{assumption}
$\forall k\leq K,\,\forall j\leq H$, we have $0\leq {\hat q^{(k)}}_j\leq R_{\max}(H+1-j) ,\, {\hat \mu^{(k)}}_j\leq C_2,\,\|\hat d^{q(k)}_j\|_{\oper},\|\hat d^{\mu(k)}_j\|_{\oper}\leq C_4$.
\end{assumption}

\begin{theorem}[Efficiency]
\label{thm:db}
Suppose $\forall k\leq K,\,\forall j\leq H$,
\begin{align*}\|\hat \mu^{(k)}_j-\mu_j \|_{L^2_b}&=\op(n^{-\alpha_1}),~\|\hat d^{\mu(k)}_j-d^{\mu}_j\|_{L^2_b}=\op(n^{-\alpha_2}),\\
\|\hat q^{(k)}_j-q_j\|_{L^2_b}&=\op(n^{-\alpha_3}),~\|\hat d^{q(k)}_j-d^{q}_j\|_{L^2_b}=\op(n^{-\alpha_4}),
\end{align*}
$\min(\alpha_1,\alpha_2)+\min(\alpha_3,\alpha_4)\geq 1/2$ and $\alpha_1,\alpha_2,\alpha_3,\alpha_4>0$. 
Then, $\sqrt{n}(\hat Z^{\DO}(\theta)-Z(\theta))\stackrel{d}{\rightarrow}\mathcal{N}(0,\var[\xi_{\MDP}])$. 
\end{theorem}

An important feature of \cref{thm:db} is that the required nuisance convergence rates are nonparametric (slower than $n^{-1/2}$) and no metric entropy condition (\eg, Donsker) is needed. This allow many types of flexible machine-learning estimators to be used.

Importantly, we experience no variance inflation due to plugging-in estimates instead of true nuisances. While usually we can expect inflation due to nuisance variance (\eg, PG \cref{eq:pg_q} generally has MSE worse than $\Theta(n^{-1/2})$ if we use an estimate $\hat q$ with a nonparametric rate), we avoid this due to the special doubly robust structure of $\xi_\MDP$.

To establish this structure -- the key step of the proof -- we show that
$\hat Z^\DO(\theta)$ is equal to 
\begin{equation}\label{eq:DRstructure}
\ts \E_n[\xi_{\MDP}]+\ts K^{-1}\sum_{k=1}^{K}\sum_{j=0}^{H}\mathrm{Bias}_{k,j}+\op(n^{-1/2}),
\end{equation}
where 
$\|\mathrm{Bias}_{k,j}\|_{2}\lnapprox\ts\|\hat \mu^{(k)}_j-\mu_j \|_{L^2_b}\|\|\hat d^{q(k)}_j-d^{q}_j\|_{L^2_b}+\|\hat d^{\mu(k)}_j-d^{\mu}_j\|_{L^2_b}\|\hat q^{(k)}_j-q_j \|_{L^2_b}+\|\hat \mu^{(k)}_{j-1}-\mu_{j-1}\|_{L^2_b}\|\hat d^{v(k)}_j-d^{v}_j\|_{L^2_b}+\|\hat d^{\mu(k)}_{j-1}-d^{\mu}_{j-1}\|_{L^2_b}\|\hat v^{(k)}_j-v_j \|_{L^2_b}$.

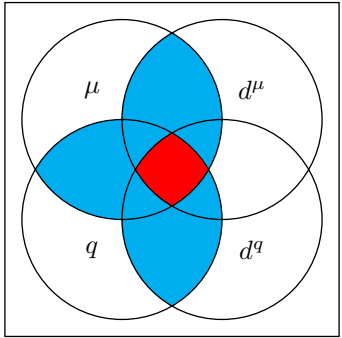
\begin{figure}
\centering
\def\mucirc{(-1/2,1/2) circle (1)}
\def\qcirc{(-1/2,-1/2) circle (1)}
\def\etacirc{(1/2,1/2) circle (1)}
\def\zetacirc{(1/2,-1/2) circle (1)}
\begin{tikzpicture}[x=1.33cm,y=1.33cm]
\begin{scope}
      \clip\mucirc;
      \fill[cyan]\etacirc;
\end{scope}
\begin{scope}
      \clip\qcirc;
      \fill[cyan]\zetacirc;
\end{scope}
\begin{scope}
      \clip\mucirc;
      \fill[cyan]\qcirc;
\end{scope}
\begin{scope}
      \clip\zetacirc;
      \clip\qcirc;
      \clip\mucirc;
      \fill[red]\etacirc;
\end{scope}
\draw[line width=.5] \qcirc (-4/5,-4/5) node [text=black] {$q$}
      \zetacirc (4/5,-4/5) node [text=black] {$d^{q}$}
      \mucirc (-4/5,4/5) node [text=black] {$\mu$}
      \etacirc (4/5,4/5) node [text=black] {$d^{\mu}$}
      (-1.67,-1.67) rectangle (1.67,1.67);
\end{tikzpicture}
\caption{Doubly robust and efficient structure of EOPPG. Every circle represents the event that the corresponding nuisance is well-specified. The cyan-shaded region represents the event that $\hat Z^{\DO}(\theta)$ is consistent. The red-shaded region represents the event that $\hat Z^{\DO}(\theta)$ is efficient (when nuisances are consistently estimated non-parametrically at slow rates).}
\label{fig:weak}
\end{figure}

After establishing this key result, \cref{eq:DRstructure}, \cref{thm:db} follows by showing that the bias term is $\op(n^{-1/2})$ and CLT. We also obtain the following from \cref{eq:DRstructure} via LLN.
\begin{theorem}[3-way double robustness]
\label{thm:db_weak}
Suppose $\forall k\leq K,\forall j\leq H,\,\|\hat \mu^{(k)}_j-\mu^{\dagger}_j \|_{L^2_b},\|\hat d^{q(k)}_j-d^{q\dagger}_j\|_{L^2_b},\|\hat q^{(k)}_j-q^{\dagger}_j\|_{L^2_b},\| \hat d^{\mu(k)}_j-d^{\mu\dagger}_j\|_{L^2_b}$ all converge to 0 in probability. Then $\hat Z^{\DO}(\theta)\to_pZ(\theta)$ if one the following hold:
$\mu^{\dagger}_j=\mu_j,d^{\mu\dagger}_j=d^{\mu}_j$; $q^{\dagger}_j=q_j,d^{q\dagger}_j=d^{q}_j$; or $\mu^{\dagger}_j=\mu_j,q^{\dagger}_j=q_j$.
\end{theorem}

That is, as long as (a) $\hat\mu,\hat d^{\mu}$ are correct, (b) $\hat q,\hat d^{q}$ are correct, or (c) $\hat \mu,\hat q$ are correct, EOPPG is still consistent. The reason the estimator is not consistent when only $\hat d^{q},\hat d^{\mu}$ are correct is because $\hat d^{v}$ is constructed using both $\hat q,\hat d^{q}$ (see \cref{eq:vhat}).

\subsection{Special Cases}\label{sec:comparison}

\begin{example}[On-policy case]
If $\bpol=\thpol$, then 
\begin{align*}
\xi_{\NMDP}&\ts=\sum_{j=0}^H((\sum_{i=j}^Hr_{i}+v_{i+1}-q_{i})g_j+d^v_j-d^q_j),\\
\xi_{\MDP}&\ts=
\sum_{j=0}^{H} (d^{\mu}_j (r_j-q_j)-d^{q}_j
      +d^{v}_j +d^{\mu}_{j-1}v_j),
\end{align*}
where $d^{\mu}_j=\E[\sum_{i=0}^{j}g_i(a_i\mid s_i) \mid a_j,s_j]$. (Recall that $q_j,d^q_j$ are functions of $\ch_{a_j}$ in NMDP but only of $(s_j,a_j)$ in MDP; similarly for $v_j,d^v_j$ and $\ch_{s_j}$ compared to just $s_j$.)

In the on-policy case, \citet{HuangJiawei2019FISt,ChengChing-An2019TCVf} propose estimators equivalent to estimating $q,d^q$ and plugging into the above equation for $\xi_{\NMDP}$. Using our results (establishing the efficiency bound and that $\xi_{\NMDP}$ is the EIF under NMDP) these estimators can then be shown to be efficient for NMDP (either under a Donsker condition or using cross-fitting instead of their in-sample estimation). These are not efficient under MDP, however, and $\xi_{\MDP}$ will still have lower variance. However, in the on-policy case, $C_1=1$, so the curse of horizon does not affect $\xi_{\NMDP}$ and since it requires fewer nuisances it might be preferable.
\end{example}

\begin{example}[Stationary infinite-horizon case]
Suppose the MDP transition and reward probabilities and the behavior and target policy ($\thpol$) are all stationary (\ie, time invariant so that $\pi=\pi_t$, $g=g_t$, $p_t=p$, \emph{etc.}).
Suppose moreover that, as $H\to\infty$ the Markov chain on the variables $\{(s_t,a_t,r_t):t=0,1,\dots\}$ is ergodic under either the behavior or target policy.
Consider estimating the derivative of the long-run average reward $Z^\infty(\theta)=\lim_{H\to\infty}Z(\theta)/H$. 
By taking the limit of $\xi_\MDP/H$ as $H\to\infty$, we obtain
\begin{align*}
\xi^\infty_\MDP\overset{\text{dist}}{=}&~d^{\mu}(s',a')(r'-q(s',a'))-\mu(s',a')d^{q}(s',a') \\
      &+\mu(s,a)d^{v}(s') +d^{\mu}(s,a)v(s'),
\end{align*}
where $(s,a,r,s',a')$ are distributed as the stationary distribution of $(s_t,a_t,r_t,s_{t+1},a_{t+1})$ under the behavior policy, $\mu(s,a)$ is the ratio of stationary distributions of $(s_t,a_t)$ under the target and behavior policies, $q(s,a)$ and $v(s)$ are the long-run average $q$- and $v$-functions under the target policy, and $d^\mu,\,d^q,\,d^v$ are the derivatives with respect to $\theta$.

It can be shown that under appropriate conditions, $\xi^\infty_\MDP$ is in fact the EIF for $Z^\infty(\theta)$ if our data were iid observations of $(s,a,r,s',a')$ from the stationary distribution under the behavior policy. If our data consists, as it does, of $n$ observations of $(H+1)$-long trajectories, then we can instead construct the estimator
\begin{align*} 
  &\ts\frac1{n(H+1)}\sum_{i=1}^n\sum_{j=0}^{H}\bigl( d^{\mu}(s^{(i)}_j,a^{(i)}_j)(r^{(i)}_j-q(s^{(i)}_j,a^{(i)}_j)) \\
      &\ts-\mu(s^{(i)}_j,a^{(i)}_j)d^{q}(s^{(i)}_j,a^{(i)}_j)
      +\mu(s^{(i)}_{j-1},a^{(i)}_{j-1})d^{v}(s^{(i)}_j)\\
      &\ts+d^{\mu}(s^{(i)}_{j-1},a^{(i)}_{j-1})v(s^{(i)}_j)\bigr),
\end{align*}
where the nuisances $\mu,d^{\mu},q,d^q$ are appropriately estimated in a cross-fold manner as in \cref{alg:EOPPG}. Following similar arguments as in \citet{KallusNathan2019EBtC}, which study infinite-horizon OPE, one can show that this extension of EOPPG maintains its efficiency and 3-way robustness guarantees as long as our data satisfies appropriate mixing conditions (which ensures it appropriately approximates observing draws from the stationary distribution). Fleshing out these details is beyond the scope of this paper.
\end{example}

\begin{example}[Logged bandit case]
If $H=0$ (one decision point), then $\xi_{\MDP}=\xi_{\NMDP}$ are both equal to
\begin{align*} \ts
  \tilde\nu_0(r_0-q_0) g_0+\E_{\thpol_0(a_0\mid s_0)}[ q_0g_0\mid s_0]. 
\end{align*}
We can construct an estimator by cross-fold estimation of $q_0$ (note the last expectation is just an integral over $\thpol(a\mid s)$ for a given $s$). While policy gradients are used in the logged bandit case in the counterfactual learning community \citep[\eg][which use the gradient $\tilde\nu_0r_0 g_0$]{swaminathan2015counterfactual}, as far as we know, no one uses this efficient estimator for the gradient even in the logged bandit case.  
\end{example}

\begin{example}
\label{ex:triply}
By \cref{thm:db_weak}, each of the following is a new policy gradient estimator that is consistent given consistent estimates of its respective nuisances:\\ 
\textbf{a) $\hat \mu=0,\hat d^{\mu}=0$}: $\E_n[\hat d^{v}_0]$,   \\
\textbf{b) $\hat q=0,\hat d^{q}=0$}:   $ \E_n[\sum_{j=0}^{H}\hat d^{\mu}_{j}r_j]$, \\
\textbf{c) $\hat d^{q}=0,\,\hat d^{\mu}=0$}: $\E_n[\sum_{j=0}^{H} \E_{\thpol}[\hat\mu_{j-1}\hat q_{j}g_j\mid s_j]]$,\\
where the inner expectation is only over $a_j\sim \thpol(a_j\mid s_j)$.
\end{example}

\subsection{Estimation of Nuisance Functions}\label{sec:algorithm2}
\label{sec:nuisance}

We next explain how to estimate the nuisances $d_j^\mu$ and $d_j^q$. The estimation of $q_j$ is covered by \citet{ChenJinglin2019ICiB,munos2008finite} and the estimation of $\mu_j$ by \citet{XieTengyang2019OOEf,KallusUehara2019}. 

\paragraph{Monte-Carlo approach.} First we explain a Monte-Carlo way to estimate $d^{q}_j, d^{\mu}_j$. We use the following theorem. 
\begin{theorem}[Monte Carlo representations of $d^{\mu}_j,\,d^{q}_j$]
\label{thm:monte_carlo}
\begin{align*}\ts
d^{q}_j&\ts=\E\left[\sum_{t=j+1}^{H}r_t\nu_{j+1:t} \sum_{\ell=j+1}^{t}g_\ell\mid a_{j},s_{j}\right],\\
\ts d^{\mu}_j&\ts=\E\left[\nu_{0:j}\sum_{\ell=0}^{j}g_\ell\mid a_j,s_j\right].
\end{align*}
\end{theorem}
Based on this result, we can simply learn $d^q_j,d^\mu_j$ using any regression algorithm. Specifically, we construct the response variables $y^{(i)}_{d^q_j}=\sum_{t=j+1}^{H}r^{(i)}_t\nu^{(i)}_{j+1:t} \sum_{\ell=j+1}^{t}g^{(i)}_\ell$, $y^{(i)}_{d^\mu_j}=\nu^{(i)}_{0:j}\sum_{\ell=0}^{j}g_\ell^{(i)}$, and we regress these on $(a^{(i)}_j,s^{(i)}_j)$.
For example, we can use empirical risk minimization:
$$\ts
\hat d^q_j=\argmin_{f\in\mathcal F}
\frac1n\sum_{i=1}^n\left(y_{d^q_j}^{(i)}-f(a_j,s_j)\right)^2,
$$
and similarly for $\hat d^\mu_j$. Examples for $\mathcal F$ include RKHS norm balls, an expanding subspace of $L_2$ (\ie, a sieve), and neural networks.
Rates for such estimators can, for example, be derived from \citet{WainwrightMartinJ2019HS:A,BartlettPeterL.2005LRc}.

A careful reader might wonder whether estimating nuisances in this way causes the final EOPPG estimator to suffer from the curse of horizon, since $\nu_{0:j}$ can be exponentially growing in $j$. However, as long as we have suitable nonparametric rates (in $n$) for the nuisances as in \cref{thm:db}, the asymptotic MSE of $\hat Z^{\DO}(\theta)$ does \emph{not} depend on the estimation error of the nuisances. These errors only appear in higher-order (in $n$) terms and therefore vanish. This is still an important concern in finite samples, which is why we next present an alternative nuisance estimation approach.

\paragraph{Recursive approach.} 
Next, we explain a recursive way to estimate $d^{q}_j, d^{\mu}_j$. This relies on the following result.
\begin{theorem}[Bellman equations of $d^{q}_j, d^{\mu}_j$ ]
\label{thm:Bellman}
\begin{align*}\ts
    d^{q}_j(s_j,a_j) &=  \E[d^{v}_{j+1}\mid s_j,a_j],\,d^{v}_j(s_j) =\E_{\thpol}[d^{q}_j+g_{j} q_{j}\mid s_{j}] \\
     d^{\mu}_j(s_j,a_j) &=\E[d^{\mu}_{j-1}\tilde\nu_{j}\mid s_j,a_j]+\mu_j g_j. 
\end{align*}
\end{theorem}
This is derived by differentiating the Bellman equations: 
\begin{align*}\ts
    q_j(s_j,a_j) &= \E[r+ v_{j+1}(s_{j+1})\mid s_j,a_j],\\
    \mu_j(s_j,a_j)&= \E[\mu_{j-1}(s_{j-1},a_{j-1})\tilde\nu_{j}|s_j,a_j]. 
\end{align*}
Following \cref{thm:Bellman}, we propose the recursive \cref{alg:gradient_q,alg:gradient_mu} that estimate $d^{q}_j$ using backwards recursion and $d^{\mu}_j$ using forward recursion.

\begin{remark}
\citet{MorimuraTetsuro2010DoLS} discussed a way to estimate the gradient of the stationary distribution in an on-policy setting. In comparison, our setting is off-policy. 
\end{remark}

\begin{remark}
We have directly estimated $d^\mu_j$. Another way is using $d^\mu_j=\tilde\nu_{j}\nabla_\theta \tilde \mu_j +\tilde \mu_j g_j$ and estimating $\nabla_\theta \tilde \mu_j$ recursively based on a Bellman equation for $\nabla \tilde \mu_j$, derived in a similar way to that for $d^\mu_j$ in \cref{thm:Bellman}. 
\end{remark}

\section{Off-policy Optimization with EOPPG}\label{sec:optimization}

Next, we discuss how to use the EOPPG estimator presented in \cref{sec:EOPPG} for off-policy optimization using projected gradient ascent and the resulting guarantees. The algorithm is given in \cref{alg:gradient}. 

\begin{algorithm}[t!]
 \caption{Estimation of  $d^{q}_j$ (Recursive way)}
 \begin{algorithmic}
 \label{alg:gradient_q}
 \STATE \textbf{Input}: $q$-estimates $\hat q_j$, hypothesis classes $\Fcal^{d^{q}_j}$
 \STATE Set $\hat d^{v}_H=\hat d^{q}_H=0$
\FOR{$j=H-1,H-2,\cdots$}
 \STATE Set $\hat d^{q}_{j}\in\argmin\limits_{f\in\Fcal^{d^{q}_j}}\sum\limits_{i=1}^n\left( \hat d^{v}_{j+1}(s^{(i)}_{j+1}) -f(s^{(i)}_j,a^{(i)}_j) \right)^2$
 \STATE Set $\hat d^{v}_{j}(s_j)=\E_{\thpol_j}[\hat d^{q}_j +\hat q_j g_j \mid s_j]$\\(by integrating/summing w.r.t $a_j\sim\thpol_j(a_j\mid s_j))$
 \ENDFOR
\end{algorithmic}
\end{algorithm}
\begin{algorithm}[t!]
 \caption{Estimation of $d^{\mu}_j$ (Recursive way)}
 \begin{algorithmic}
 \label{alg:gradient_mu}
 \STATE \textbf{Input}:  $\mu$-estimates $\hat \mu_j$, hypothesis classes $\Fcal^{d^{\mu}_j}$
 \STATE Set $\hat d^{\mu}_{0}=\nu_0g_0$
\FOR{$j=1,2,\cdots$}
 \STATE  Set $\hat d^{\mu}_j=\argmin_{f\in \in \Fcal^{d^{\mu}_j}}
   \sum_{i=1}^n\bigl(f(s^{(i)}_{j},a^{(i)}_{j})$\\\qquad\qquad\qquad\quad$-\tilde\nu^{(i)}_{j}\hat d^{\mu}_{j-1}(s^{(i)}_{j-1},a^{(i)}_{j-1}) -\hat \mu^{(i)}_{j}g^{(i)}_j \bigr)^2$
 \ENDFOR
\end{algorithmic}
\end{algorithm}
\begin{algorithm}[t!]
 \caption{Off-policy projected gradient ascent}
 \begin{algorithmic}
 \label{alg:gradient}
 \STATE {\bfseries Input:} An initial point $\theta_1\in\Theta$ and step size schedule $\alpha_t$ 
\FOR{$t=1,2,\cdots$}
 \STATE $\tilde \theta_{t+1} =\theta_{t}+\alpha_t \hat Z^{\DO}(\theta_t)$
 \STATE $\theta_{t+1} = \operatorname{Proj}_{\Theta}(\tilde \theta_{t+1})$
 \ENDFOR
\end{algorithmic}
\end{algorithm}

Then, by defining an error $B_t=\hat Z^\DO(\theta_t)-Z(\theta_t)$, we have the following theorem. 
\begin{theorem}\label{thm:optimization}
Assume the function $J(\theta)$ is differentiable and $M$-smooth in $\theta$, $M<1/(4\alpha_t)$, and $\tilde \theta_{t}=\theta_{t}$.\footnote{This means all iterates remain in $\Theta$ so the projection is identity. This is a standard condition in the analysis of non-convex optimization method that can be guaranteed under certain assumptions; see \citet{khamaru18a,NesterovYurii2006CroN}.} Set $J^{*}=\max_{\theta \in \Theta} J(\theta)$. Then, $\{\theta_t\}_{t=1}^{T}$ in \cref{alg:gradient} satisfies
\begin{align}\ts\notag
\frac{1}{T}\ts\sum_{t=1}^{T} \alpha_t &\| Z(\theta_t)\|^2_{2} 
   \ts\leq\frac{4(J^{*}-J(\theta_1))}{T}+\frac{3}{T}\sum_{t=1}^{T}\alpha_t\|B_t\|^2_{2}.  \nonumber 
\end{align}
\end{theorem}

\Cref{thm:optimization} gives a bound on the average derivative. That is, if we let $\hat\theta$ be chosen at random from $\{\theta_t\}_{t=1}^T$ with weights $\alpha_t$, then
via Markov's inequality,
$$\ts
Z(\hat\theta)=\mathcal O_p(\frac{4}{T}(J^{*}-J(\theta_1))+\frac{3}{T}\sum_{t=1}^{T}\alpha_t\|B_t\|^2_{2}).
$$
So as long as we can bound the error term $\sum_t \alpha_t\|B_t\|^2_{2}/T$, we have that $\hat\theta$ becomes a near-stationary point.

This error term comes from the noise of the EOPPG estimator. A heuristic calculation based on \cref{thm:db} that ignores the fact that $\theta_t$ is actually random would suggest
\begin{align*}\ts 
   \|B_t\|^2_{2} &\lnapprox \operatorname{trace}(\var[\xi_{\MDP}])+\op(1/n) \nonumber \\
   &\lnapprox \frac{DC_2R^2_{\max}G_{\max}(H+1)^2(H+2)^2}{n}+\op(1/n). 
\end{align*}

To establish this formally, we recognize that $\theta_t$ is a random variable and bound the \emph{uniform} deviation of EOPPG over all $\theta\in\Theta$. We then obtain the following high probability bound on the cumulative errors.  
\begin{theorem}[Bound for cumulative errors]
\label{thm:error}
Suppose the assumptions of \cref{thm:db} hold, that $\theta \to \xi_{\MDP,j}$ is almost surely differentiable with derivatives bounded by $L$ for $j\in \{1,\cdots,D\}$, where $\xi_{\MDP,j}$ is a $j$-th component of $\xi_{\MDP}$, and that $\Theta$ is compact with diameter $\Upsilon$. 

Then, for any $\delta$, there exists $N_{\delta}$ such that $\forall n\geq N_{\delta}$, we have that, with probability at least $1-\delta$,
\begin{align*}\ts 
    &\ts\frac{1}{T}\sum_t\|B_t\|^2_{2}  \lnapprox U_{n,T,\delta},\\
    &\ts U_{n,T,\delta}=\frac{D(L^2 D\Upsilon^2+C_2 G_{\max}R^2_{\max}(H+1)^2(H+2)^2 \log(TD/\delta))}{n}. 
\end{align*}
\end{theorem}
This shows that, by letting $T=n^{\beta}$ ($\beta>1$) be sufficiently large, we can obtain $Z(\hat\theta)={\mathcal O}_p(H^4C_2\log(n)/n)$ for $\hat\theta$ chosen at random from $\{\theta_t\}_{t=1}^T$ as above. 
Note that if we had used other policy gradient estimators such as (step-wise) REINFORCE, PG as in \cref{eq:pg_q}, or off-policy variants of the estimators of \citet{HuangJiawei2019FISt,ChengChing-An2019TCVf}, then the term $C^{H}_1$ would have appeared in the bound and the curse of horizon would have meant that our learned policies would not be near-stationary for long-horizon problems.

\begin{remark}
Many much more sophisticated gradient-based optimization methods equipped with our EOPPG gradient estimator can be used in place of the vanilla projected gradient ascent in \cref{alg:gradient}. We refer the reader to \citet{JainPrateek2017NOfM} for a review of non-convex optimization methods.
\end{remark}

\paragraph{The concave case.}
The previous results study the guarantees of \cref{alg:gradient} in terms of convergence to a stationary point, which is the standard form of analysis for non-convex optimization. If we additionally assume that $J(\theta)$ is a \emph{concave} function then we can see how the efficiency of EOPPG translates to convergence to an optimal solution in terms of the \emph{regret} compared to the optimal policy. 
In this case we set $\hat \theta =\frac{1}{T}\sum_{t=1}^{T} \theta_t$, for which we can prove the following:
\begin{theorem}[Regret  bound]
\label{thm:convexity}
Suppose the assumptions of Theorem \ref{thm:error} hold, that $ J(\theta)$ is a concave function, and that $\Theta$ is convex. For a suitable choice of $\alpha_t$ we have that, for any $\delta$, there exists $N_{\delta}$ such that $\forall n\geq N_{\delta}$, we have that, with probability at least $1-\delta$, 
\begin{align*}\ts
&J^{*}-J(\hat \theta) \lnapprox \Upsilon{ \frac{\sup_{\theta \in \Theta} \|Z(\theta)\|_{2}+\sqrt{U_{n,T,\delta}}}{\sqrt T}}.
\end{align*}
\end{theorem}

Here, the first term is the optimization error if we knew the true gradient $Z(\theta)$. The second term is the approximation error due to the error in our estimated gradient $\hat Z^{\DO}(\theta)$. 
When we set $T=n^{\beta}$ ($\beta>1$), the final regret bound is 
\begin{align*}\ts
\bigO_{p}\left(\Upsilon R_{\max}H^2\sqrt{{D\beta G_{\max}C_2 \log(nD/\delta))}}/\sqrt{n}\right). 
\end{align*}
The regret's horizon dependence is $H^2$. This is a crucial result since the regret with polyomial horizon dependence is a desired result in RL \citep{pmlr-v75-jiang18a}. 
Again, if we had used other policy gradient methods, then an exponential dependence via $C^{H}_1$ would appear. Moreover, the regret depends on $C_2$, which corresponds to a concentrability coefficient \citep{antos2008learning}. 

\begin{remark} Recent work studies the global convergence of online-PG algorithms without concavity \citep{BhandariJalaj2019GOGF,AgarwalAlekh2019OaAw}. This may be applicable here, but our setting is completely off-policy and therefore different and requiring future work. Notably, the above focus on optimization rather than PG variance reduction. In a truly off-policy setting, the available data is limited and statistical efficiency is crucial and is our focus here.
\end{remark}

\begin{remark}[Comparison with other results for off-policy policy learning]
In the logged bandit case ($H=0$), the regret bound of off-policy learning via exhaustive search (non-convex) optimization is $\bigO_{p}(\sqrt{\tau(\Pi)\log(1/\delta)/n})$, where $\tau(\Pi)$ represents the complexity of the hypothesis class \citep{ZhouZhengyuan2018OMPL,FosterDylanJ.2019OSL}. In this bandit case, the nuisance functions of the EIF do not depend on the policy itself, making this analysis tractable. However, for our RL problem ($H>0$), nuisance functions depend on the policy; thus, these techniques do not extend directly. \citet{NieXinkun2019LWP} do extend these types of regret results to an RL problem but where the problem has a special when-to-treat structure, not the general MDP case.
\end{remark}

\begin{figure}[t!]
     \centering
         \centering
         \includegraphics[width=0.7\linewidth]{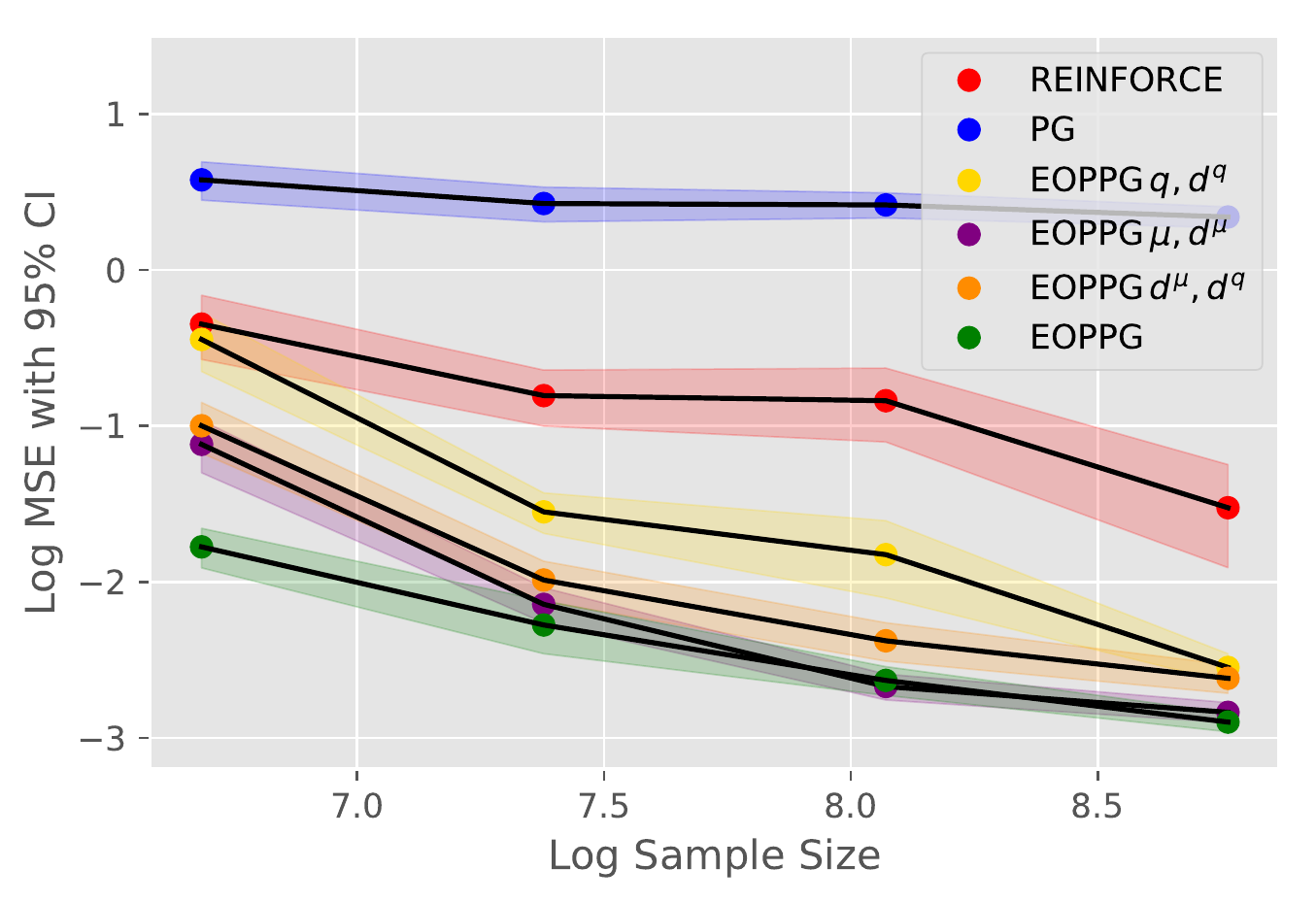}
          \vspace{-0.5cm}
         \caption{Comparison of MSE of gradient estimation}
         \label{fig:MSE}
\end{figure}

\begin{figure}[t!]
     \centering
         \centering
         \includegraphics[width=0.7\linewidth]{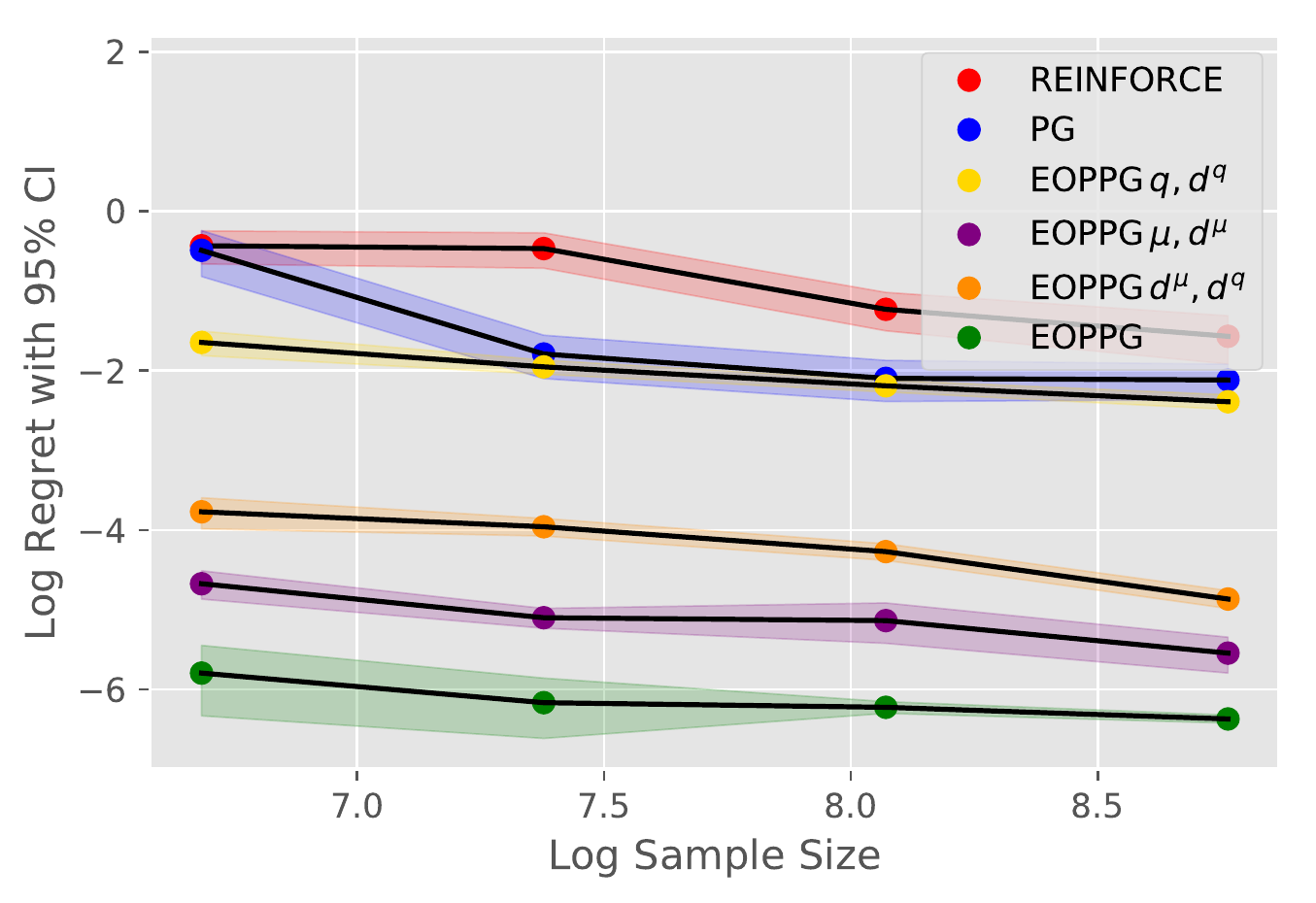}
             \vspace{-0.5cm}
         \caption{Comparison of regret after gradient ascent}
         \label{fig:regret}
\end{figure}

\section{Experiments}

We conducted an experiment in a simple environment to confirm the theoretical guarantees of the proposed estimator. More extensive experimentation remains future work. The setting is as follows. Set $\mathcal{S}_t=\Rl,\,\mathcal{A}_t=\Rl,s_0=0$. Then, set the transition dynamics as $s_t=a_{t-1}-s_{t-1}$, the reward as $r_{t}=-s_t^2$, the behavior policy as $\bpol_t(a\mid s)=\mathcal{N}(0.8s,0.2^2)$, the policy class as $\mathcal{N}(\theta s,0.2^2)$, and horizon as $H=50$. Then, $\theta^*=1$ with optimal value $J^*=-1.96$, obtained by analytical calculation. 
We compare REINFORCE (\cref{eq:step_reinforce}), PG (\cref{eq:pg_q}), and EOPPG with $K=2$. Nuisances functions $q,\,\mu,\,d^{q},\,d^{\mu}$ are estimated by polynomial sieve regressions \citep{ChenXiaohong2007C7LS}. Additionally, to investigate 3-way double robustness, we consider corrupting the nuisance models by adding noise $\mathcal{N}(0,1)$; we consider thus corrupting $(q,d^{q})$, $(\mu,d^{\mu})$, or $(d^{\mu},d^{q})$.

First, in \cref{fig:MSE}, we compare the MSE of these gradient estimators at $\theta=1.0$ using $100$ replications of the experiment for each of $n=800,1600,3200,6400$. As can be seen, the performance of EOPPG is superior to existing estimators in terms of MSE, validating our efficiency results. We can also see that the EOPPG with misspecified models still converges, validating our 3-way double robustness results. 

Second, in in \cref{fig:regret}, we apply a gradient ascent as in \cref{alg:gradient} with $\alpha_t=0.15$ and $T=40$. We compare the regret of the final policy, \ie, $J(\theta^{*})-J(\hat \theta_{40})$, using $60$ replications of the experiment for each of $n=200,400,800,1600$. Notice that the lines decrease roughly as $1/\sqrt{n}$ but because of the large differences in values, the lines only appear somewhat flat. This shows that the efficiency and 3-way double robustness translate to good regret performance, as predicted by our policy learning analysis.

\section{Conclusions}

We established an MSE efficiency bound of order $\mathcal{O}(H^4/n)$ for estimating a policy gradient in an MDP in an off-policy manner. We proposed an estimator, EOPPG, that achieves the bound, enjoys 3-way double robustness, and leads to regret dependence of order $H^2/\sqrt{n}$ when used for policy learning. Notably, this is much smaller than other approaches, which incur exponential-in-$H$ errors.

This paper is only a first step toward efficient and effective off-policy policy gradients in MDPs. Remaining questions include how to estimate $d^{q},\,d^{\mu}$ in a large-scale environments, the performance of more practical implementations that alternate in updating $\theta$ and nuisance estimates with only one gradient update, and extending our theory to the deterministic policy class as in \citet{silver14}. 

\bibliographystyle{chicago}
\bibliography{ic,pfi3}
\appendix

\section{Notation}
\label{sec:notation}
\begin{table}[H]
    \caption{Notation}
    \centering
    \begin{tabular}{l|l}
    $\bpol$,  $\thpol$     &  Behavior policy,\,Evaluation policy \\
    $H$ & Horizon \\ 
    $T$ & Final optimization step \\ 
    $\nabla,\,\nabla_{\theta} $ & Derivative w.r.t $\theta$ \\
    $\ch_{a_t}$ &   History up to $a_t$, $(s_0,a_0,\cdot,s_t,a_t)$ (Rewards are excluded)\\
    $\ch_{s_t}$ &   History up to $s_t$, $(s_0,a_0,\cdot,s_t)$ (Rewards are excluded)\\
    $\ck_{a_t}$ &   History up to $a_t$, $(s_0,a_0,r_0,\cdot,a_t)$ (Rewards are included) \\
    $\ck_{s_t}$ &   History up to $s_t$, $(s_0,a_0,r_0,\cdot,s_t,a_t)$ (Rewards are included) \\
    $\E_{\pi}[\cdot] $ & Expectation is taken w.r.t policy $\pi$ \\
     $\E_{n}[\cdot] $ & Empirical approximation \\
    $J(\theta)$ & Value of $\pi$ \\
    $Z(\theta)$ & $\nabla J(\theta) $\\
    MDP, NMDP & DAG MDP,\,Tree MDP\\ 
    $q_t(s,a)$ & State action value function at $t$ \\
    $v_t(s)$ & State Value function at $t$  \\
    $\nu_{0:t}(\ch_{a_t})$ & $\prod_{0=i}^{H}\thpol_i/\bpol_i$ \\
    $\nu_{a:b}$ & $\prod_{i=a}^{b}\thpol_i/\bpol_i$ \\
     $\tilde\nu_{t}$ & $\thpol_t/\bpol_t$ \\
     $\tilde\mu_{t}$ & $p_{\thpol}(s_t)/p_{\bpol}(s_t)$ \\
    $\mu_t(s,a)$ & Ratio of marginal distribution at $t$\\
    $d^{\nu}_t(s,a),d^{\mu}_t(s,a)$  &  $\nabla \mu_t(s,a),\nabla \mu_t(s,a$\\
    $d^{q}_t(s,a),d^{v}_t(s) $  &  $\nabla q_t(s,a),\nabla v_t(s,a)$\\
    $\otimes g$ &  $gg^{\top}$  \\
    $g_t$ & Score function of the policy: $\nabla \log \thpol_t(a_t\mid s_t)$ \\ 
    $C_1,\,C_2\,C'_1$ &  Distribution mismatch constants \\ 
    $R_{\max}$ & $0\leq r_t \leq R_{\max} $\\
    $G_{\max}$ & $\|g_t\|_{\oper}\leq G_{\max}$ \\
    $\mathcal{F}$ & Function class  \\ 
    $\|\cdot \|_{\mathrm{op}}$ & Operator norm  \\
    $\|\cdot \|_{L^2_b}$ & $L^2$-integral norm with respect to the behavior policy\\
    $\|\cdot \|_{2}$ & Euclidean norm  \\
    $\Theta$ & Parameter space \\
    $A \preceq B$ & $B-A$ is a semi-positive definite matrix  \\
    $\mI_{D\times D}$ & Identity matrix \\
    $\xi_{\MDP},\xi_{\NMDP}$ &  Efficient influence functions (IFs) of $Z(\theta)$ under MDP and NMDP  \\ 
    $\mathcal{U}_k$  &  $k$-th partitioned data \\  
    $\mathcal{Z}_k$  & Data except for $\mathcal{U}_k$\\ 
    $\Proj$ & projection  \\
    $\Theta$ & Parameter space \\ 
    $\Upsilon$ & Diameter of $\Theta$ \\
    $D$ & Dimension of $\theta$ \\
    $\langle \cdot , \cdot \rangle$ & Inner product $a^{\top}b$ \\
    $a \lnapprox b $ & Inequality up to absolute constant \\
    $\mathcal{I}$ & Identiry matrix 
    \end{tabular}
 \label{tab:notation}
\end{table}
\newpage

\section{Semiparametric Theory for Off-Policy RL: A General Conditioning Framework}

\label{sec:semipara}
Here, we summarize a general framework to obtain efficiency bounds and efficient influence functions for various quantities of interest under NMDP or MDP, which we then use in order to derive these for the policy gradient case. First, we present the framework in generality. Then, we show how to use this framework to re-derive the efficiency bounds and efficient influence functions for OPE of \citet{KallusUehara2019}, who derived it for the first time but from scratch. Our proofs for our policy gradient case in the subsequent sections use the observations from this section.

\subsection{General Conditioning Framework}

\subsubsection{General Semiparametric Inference}

Consider observing $n$ iid observations $O_i\sim P$ from some distribution $P$. We are interested in the estimand $R(P)$ where the unknown $P$ is assumed to live in some (nonparametric) model $P\in\mathcal M$ and $R:\mathcal M\to\R D$. Estimators of this estimand are functions of the data, $\hat R=\hat R(O^{(1)},\dots,O^{(n)})$.
Regular estimators are, roughly speaking, those for which the distribution of $\sqrt n(\hat R-R(P))$ converges to a limiting distribution in a locally uniform sense in $\mathcal M$ \citep[Chapter 25]{VaartA.W.vander1998As}. 
Under certain differentiability conditions on $R(\cdot)$, the efficiency bound is the smallest asymptotic MSE (the second moment of the distributional limit of $\sqrt n(\hat R-R(P))$) among all regular estimators $\hat R$ \citep[Theorem 25.20]{VaartA.W.vander1998As}, which also lower bounds the limit infimum of $n\E[(\hat R-R(P))^2]$ via Fatou's lemma. The efficiency bound even lower bounds the limit infimum of the MSE of \emph{any} estimator in a local asymptotic minimax sense \citep[Theorem 25.21]{VaartA.W.vander1998As}. In particular, the efficiency bound is given by $\var_P[\phi^*(O)]$ for some function $\phi^*(O)$.

Asymptotically linear estimators $\hat R$ are ones for which there exists a function $\phi(O)$ such that $\hat R=\E_n\phi+\op(n^{-1/2})$, $\E\phi=R(P)$.\footnote{Note that conventionally one restricts $\E\phi=0$ and writes $\hat R-R(P)=\E_n\phi+\op(n^{-1/2})$, but we deviate slightly here for clearer and more succinct presentation in the main text.} The function $\phi$ is known as the influence function of $\hat R$. Clearly, the asymptotic MSE of $\hat R$ is $\var_P[\phi(O)]$. Thus, an asymptotic linear estimator would be efficient if its influence function were $\phi^*$, which is called the \emph{efficient influence function}. In fact, under the same differentiability conditions on $R(\cdot)$, efficient (regular) estimators are exactly those with the influence function $\phi^*$ \citep[Theorem 25.23]{VaartA.W.vander1998As}. Under certain regularity, the set of influence functions (minus $R(P)$) is equal to the set of pathwise derivatives of $R(\cdot)$ at $P$, and the function $\phi^*$ is exactly given by that with minimal $L_2$ norm among this set \citep{klaassen1987consistent,bickel1993efficient}. Thus, $\phi^*$ can be gotten by a projection of \emph{any} influence function, which is a generic recipe for deriving the efficient influence function and the efficiency bound.

\subsubsection{A Conditioning Framework for Nonparametric Factorable Models}

We now summarize how additional graphical structure on the variable $O$ can further simplify the above recipe for deriving the efficient influence function in a particular class of models, which includes the MDP and NMDP models. Suppose each observation $O$ has $J$ component variables, $O=(O_1,\dots,O_J)$. Suppose moreover that we have some tree on the nodes $[J]=\{1,\dots,J\}$ described by the parentage relationship $\Pa:[J]\to 2^{[J]}$ mapping a node to its parents and such that $P$ satisfies the factorization
\begin{equation}\label{eq:nonparamfactorization}
P(O)=\prod_{j=1}^J P_j(O_j\mid O_{\Pa(j)}). 
\end{equation}
Consider the nonparametric model of all distributions that satisfy this factorization
$$
\mathcal M=\braces{Q\ :\ Q(O)=\prod_{j=1}^J Q_j(O_j\mid O_{\Pa(j)})
~~\text{for some conditional distributions $Q_j$}
}.
$$
Then, a standard result (see \citealp{LaanMarkJ.vanDer2003UMfC}, \citealp[\S A.7]{vanderLaanMarkJ2011TLCI}) is that, given any $\phi$ that is a valid influence function for $R(P)$ in $\mathcal M$, the efficient influence function for $R(P)$ is given by
$$
\phi^*(O)-R(P)=\sum_{j=1}^J\prns{
\E[\phi(O)\mid O_j,O_{\Pa(j)}]
-
\E[\phi(O)\mid O_{\Pa(j)}]
}.
$$
This arises due to the above-mentioned projection interpretation of the efficient influence function.

Now, suppose that the estimand only depends on a particular part of the factorization:
\begin{equation}\label{eq:factoredestimand}
R(Q)=R(Q')~~\text{whenever $Q_j=Q'_j$ for all $j\in J_0$},
\end{equation}
for some index set $J_0\subseteq[J]$.
That is, $R(Q)$ is only a function of $Q_{J_0}=(Q_j)_{j\in J_0}$ and is independent of $Q_{J_0^C}=(Q_j)_{j\notin J_0}$. 
Consider the model where we assume that $P_j$ is \emph{known} for every $j\notin J_0$,
$$
\mathcal M_0=\braces{Q\ :\ Q(O)=\prod_{j=1}^J Q_j(O_j\mid O_{\Pa(j)})
~~\text{for some $Q_{J_0}$ and $Q_{J^C_0}=P_{J^C_0}$}
}.
$$
Then, as long as $R(\cdot)$ satisfies \cref{eq:factoredestimand}, its efficient influence function under $\mathcal M$ and $\mathcal M_0$ must be the same (similarly for the efficiency bound).

Combining the above observations, we have that if (a) our model satisfies the nonparametric factorization as in \cref{eq:nonparamfactorization} and (b) our estimand only depends on some subset $J_0$ of the factorization as in \cref{eq:factoredestimand}, then given any $\phi$ that is a valid influence function for $R(P)$ in $\mathcal M_0$, the efficient influence function for $R(P)$ under $\cm$ is in fact also just given by
\begin{equation}\label{eq:factoredEIF}
\phi^*(O)-R(P)=\sum_{j\in J_0^C}\prns{
\E[\phi(O)\mid O_j,O_{\Pa(j)}]
-
\E[\phi(O)\mid O_{\Pa(j)}]
}.
\end{equation}

\subsection{Application to Off-Policy RL}

In off-policy RL, our data are observations of trajectories $\traj=(s_0,a_0,r_0,\dots,s_T,a_H,r_H,s_{H+1})$ generated by the behavior policy. Here $\traj$ stands for a single observation (above $O$ in the general case) and $s_t,a_t,r_t$ are individual components (above $O_j$ in the general case).
Moreover, in the MDP model, the data-generating distribution satisfies a factorization like \cref{eq:nonparamfactorization}:
$$p_{\bpol}(\traj)=p_0(s_0)\prod_{t=0}^H\bpol_t(a_t\mid s_t)p_t(r_t\mid s_t,a_t)p_t(s_{t+1}\mid s_t,a_t).$$
Finally, we have that off-policy quantities such as the policy value and policy gradient for $\thpol$ are independent of the behavior policy, that is, satisfy \cref{eq:factoredestimand} where $J_0^C$ corresponds to the $\bpol_t(a_t\mid s_t)$ part in the factorization above. Here, the model $\mathcal M_0$ would correspond to the model where the behavior policy is known (and indeed the efficiency bound is independent of whether it is known or not).

Similarly, in the NMDP model we have an alternative factorization, where each node's parent set is much larger:
$$
p_{\bpol}(\traj)=
p_0(s_0)\prod_{t=0}^H\bpol_t(a_t\mid \ch_{s_t})p_t(r_t\mid \ch_{a_t})p_t(s_{t+1}\mid \ch_{a_t}).
$$ 
Again, off-policy quantities of interest are independent of of the behavior policy.

These observations imply that in order to derive the efficient influence function (and hence the efficiency bound) for any appropriate off-policy quantity, we simply need to identify one valid influence function in $\mathcal M_0$ and then compute \cref{eq:factoredEIF}. This is exactly the approach we take in our proofs for the policy gradient.

Before proceeding to our proofs, which for the first time derive the efficiency bounds for off-policy gradients, as an illustrative case we first show how we can use this framework to derive the efficient influence functions and efficiency bounds for $J(\theta)$ under MDP and NMDP, which was first derived by \citet{KallusUehara2019}.

\begin{example}[Off-policy evaluation in MDP]
First we derive the efficient influence function. Under the model $\mathcal M_0$ where the behavior policy is known we know that $J(\theta)=\E\bracks{\sum_{t=0}^{H}\nu_{0:t} r_t}$ and therefore $\hat J(\theta)=\E_n\bracks{\sum_{t=0}^{H}\nu_{0:t} r_t}$ is a consistent linear estimator for $J(\theta)$. Hence, $\phi(\traj)=\bracks{\sum_{t=0}^{H}\nu_{0:t} r_t}$ must be a valid influence function.
Plugging into the right-hand side of \cref{eq:factoredEIF}, we obtain:
\begin{align*}\ts
       &\sum_{j=0}^{H}\left\{\E[ \sum_{t=0}^{H} \nu_{0:t} r_t\mid r_j,s_j,a_j]-\E[ \sum_{t=0}^{H} \nu_{0:t} r_t\mid s_j,a_j]+ \E[ \sum_{t=0}^{H} \nu_{0:t} r_t|s_j,a_{j-1},s_{j-1}]-\E[ \sum_{t=0}^{H} \nu_{0:t} r_t|a_{j-1},s_{j-1}]\right \}  \\
       &=\sum_{j=0}^{H}\{\E[\nu_{0:j}|s_j,a_j]r_j-\E[\nu_{0:j} r_j|s_j,a_j]+\E[\sum_{t=j}^{H}\nu_{0:t} r_t\mid s_j,a_j,s_{j-1}]- \E[\sum_{t=j}^{H}\nu_{0:t} r_t\mid s_{j-1},a_{j-1}]\}\\
          &=\sum_{j=0}^{H}\{\mu_j r_j+\E[\sum_{t=j}^{H}\nu_{0:t} r_t\mid s_j,a_j,s_{j-1}]- \E[\sum_{t=j-1}^{H}\nu_{0:t} r_t\mid s_{j-1},a_{j-1}]-\E[\nu_{0:H} r_H|s_T,a_H]\}\\
               &=\sum_{j=0}^{H}\{\mu_j r_j+\E[\sum_{t=j}^{H}\nu_{0:t} r_t\mid s_j,a_{j-1},s_{j-1}]- \E[\sum_{t=j}^{H}\nu_{0:t} r_t\mid s_{j},a_{j}]\} \\
               &=\sum_{j=0}^{H}\{\mu_j r_j+\E[\nu_{0:j-1}|s_j,a_{j-1},s_{j-1}]\E[\sum_{t=j}^{H}\nu_{j:t} r_t\mid s_j]- \E[\nu_{0:j} \mid s_{j},a_{j}]\E[\sum_{t=j}^{H}\nu_{j+1:t} r_t\mid s_{j},a_{j}]\} -J(\theta)\\
                  &=\sum_{j=0}^{H}\{\mu_j r_j+\E[\nu_{0:j-1}|s_j,a_{j-1},s_{j-1}]\E[\sum_{t=j}^{H}\nu_{j:t} r_t\mid s_j]- \E[\nu_{0:j} \mid s_{j},a_{j}]\E[\sum_{t=j}^{H}\nu_{j+1:t} r_t\mid s_{j},a_{j}]\} -J(\theta)\\
                 &=v_0(s_0) +\sum_{j=0}^{H}\mu_j(s_j,a_j)\{ r_j+ v_{j+1}(s_{j+1})-q_j(s_j,a_j)\}-J(\theta).
    \end{align*} 
And therefore the efficient influence function is
$$
\phi^*(\traj)=v_0(s_0) +\sum_{j=0}^{H}\mu_j(s_j,a_j)\{ r_j+ v_{j+1}(s_{j+1})-q_j(s_j,a_j)\}.
$$
The efficiency bound is given by its variance.
This matches \citet{KallusUehara2019}. 
\end{example}

\begin{example}[Off-policy evaluation in NMDP]
We repeat the above in the NMDP case. Again, we know that $\hat J(\theta)=\E_n\bracks{\sum_{t=0}^{H}\nu_{0:t} r_t}$ is still a consistent linear estimator for $J(\theta)$. Hence, $\phi(\traj)=\bracks{\sum_{t=0}^{H}\nu_{0:t} r_t}$ must be a valid influence function.
Plugging into the right-hand side of \cref{eq:factoredEIF}, we obtain:
    \begin{align*}\ts
       &\sum_{j=0}^{H}\left\{\E[ \sum_{t=0}^{H} \nu_{0:t} r_t\mid r_j,h_{a_j}]-\E[ \sum_{t=0}^{H} \nu_{0:t} r_t\mid \ch_{a_j}]+ \E[ \sum_{t=0}^{H} \nu_{0:t} r_t|s_j,\ch_{a_{j-1}}]-\E[\sum_{t=0}^{H} \nu_{0:t} r_t|\ch_{a_{j-1}}]\right \}  \\
       &=\sum_{j=0}^{H}\{\E[\nu_{0:j}|\ch_{a_j}]r_j-\E[\nu_{0:j} r_j|\ch_{a_j}]+\E[\sum_{t=j}^{H}\nu_{0:t} r_t\mid s_j,\ch_{a_{j-1}}]- \E[\sum_{t=j}^{H}\nu_{0:t} r_t\mid \ch_{a_{j-1}}]\}\\
          &=\sum_{j=0}^{H}\{\nu_{0:j} r_j+\E[\sum_{t=j}^{H}\nu_{0:t} r_t\mid \ch_{s_{j}}]- \E[\sum_{t=j-1}^{H}\nu_{0:t} r_t\mid \ch_{a_{j-1}}]-\E[\nu_{0:H} r_H|\ch_{a_H}]\}\\
        &=\sum_{j=0}^{H}\{\nu_{0:j} r_j+\E[\sum_{t=j}^{H}\nu_{0:t} r_t\mid \ch_{s_j}]- \E[\sum_{t=j}^{H}\nu_{0:t} r_t\mid \ch_{a_{j}}]\} -J(\theta)\\
               &=\sum_{j=0}^{H}\{\nu_{0:j} r_j+\E[\nu_{0:j-1}|\ch_{s_j}]\E[\sum_{t=j}^{H}\nu_{j:t} r_t\mid \ch_{s_j}]- \E[\nu_{0:j} \mid \ch_{a_{j}}]\E[\sum_{t=j}^{H}\nu_{j+1:t} r_t\mid \ch_{a_{j}}]\} -J(\theta)\\
                 &=v_0(s_0) +\sum_{j=0}^{H}\nu_{0:j}\{ r_j+ v_{j+1}(\ch_{s_{j+1}})-q_{j}(\ch_{a_j})\}-J(\theta).
    \end{align*}
And therefore the efficient influence function is
$$
\phi^*(\traj)=v_0(s_0) +\sum_{j=0}^{H}\nu_{0:j}\{ r_j+ v_{j+1}(\ch_{s_{j+1}})-q_{j}(\ch_{a_j})\}.
$$
The efficiency bound is given by its variance.
This matches \citet{jiang,thomas2016,KallusUehara2019}. 
\end{example}

\section{Proofs}
\label{sec:proof}

\begin{proof}[Proof of Theorem  \ref{thm:mdp}]

\textbf{Part 1: deriving the efficient influence function.} We use the general framework from \cref{sec:semipara}. 
Let $\overline g_t=\sum_{i=0}^{t}g_i$.
Noting that $Z(\theta)=\E\left[\sum_{t=0}^{H}r_t\nu_{0:t} \overline g_t\right]$, we see that $\sum_{t=0}^{H}r_t\nu_{0:t} \overline g_t$ is an influence function for $Z(\theta)$ in the model where the behavior policy is known. Plugging it into the right-hand-side of \cref{eq:factoredEIF}, we obtain
\begin{align*}\ts 
    & \E\left[\sum_{t=0}^{H}r_t\nu_{0:t} \overline g_t\right]  \\
    &=\sum_{j=0}^{H}\{ \E\left[\sum_{t=0}^{H}r_t\nu_{0:t} \overline g_t\mid r_j,s_j,a_j\right]-\E\left[\sum_{t=0}^{H}r_t\nu_{0:t} \overline g_t\mid s_j,a_j\right]+\E\left[\sum_{t=0}^{H}r_t\nu_{0:t} \overline g_t\mid s_j,a_{j-1},s_{j-1}\right] \\
    &-\E\left[\sum_{t=0}^{H}r_t\nu_{0:t} \overline g_t\mid a_{j-1},s_{j-1}\right] \} \\ 
     &=\sum_{j=0}^{H}\{ \E\left[\nu_{0:j} \overline g_j\mid s_j,a_j\right]r_j-\E\left[\nu_{0:j} \overline g_j r_j\mid s_j,a_j\right]+\E\left[\sum_{t=j}^{H}r_t\nu_{0:t} \overline g_t\mid s_j,a_{j-1},s_{j-1}\right] \\
    &-\E\left[\sum_{t=j}^{H}r_t\nu_{0:t} \overline g_t\mid a_{j-1},s_{j-1}\right] \} \\
    &=\sum_{j=0}^{H}\left\{ \E\left[\nu_{0:j} \overline g_j\mid s_j,a_j\right]r_j-\E\left[\sum_{t=j}^{H}r_t\nu_{0:t} \overline g_t\mid s_j,a_{j-1},s_{j-1}\right]+\E\left[\sum_{t=j}^{H}r_t\nu_{0:t} \overline g_t\mid a_{j},s_{j}\right] \right\}-Z(\theta). 
\end{align*}
Then, by substituting $\overline g_t=\sum_{i=0}^{t}g_i$, we obtain
\begin{align*}\ts 
& \E\left[\sum_{t=j}^{H}r_t\nu_{0:t} \left\{\sum_{i=0}^{t}g_i\right\}\mid a_{j},s_{j}\right] \\
&=\E\left[\sum_{t=j}^{H}r_t\nu_{0:t} \left\{\sum_{i=j+1}^{t}g_i\right\}\mid a_{j},s_{j}\right] +\E\left[\sum_{t=j}^{H}r_t\nu_{0:t} \left\{\sum_{i=0}^{j}g_i\right\}\mid a_{j},s_{j}\right] 
\\
&=\E[\nu_{0:j} |a_j,s_j] \E\left[\sum_{t=j}^{H}r_t\nu_{j+1:t} \left\{\sum_{i=j+1}^{t}g_i\right\}\mid a_{j},s_{j}\right] +\E\left[\nu_{0:j}\left\{\sum_{i=0}^{j}g_i\right\} \mid a_j,s_j\right]\E\left[\sum_{t=j}^{H}r_t\nu_{j+1:t} \mid a_{j},s_{j}\right]. 
\end{align*}
In addition, 
\begin{align*}\ts 
& \E\left[\sum_{t=j}^{H}r_t\nu_{0:t} \left\{\sum_{i=0}^{t}g_i\right\}\mid s_j,a_{j-1},s_{j-1}\right] \\
&= \E\left[\sum_{t=j}^{H}r_t\nu_{0:t} \left\{\sum_{i=j+1}^{t}g_i\right\}\mid s_j,a_{j-1},s_{j-1}\right] + \E\left[\sum_{t=j}^{H}r_t\nu_{0:t} \left\{\sum_{i=0}^{j-1}g_i\right\}\mid s_j,a_{j-1},s_{j-1}\right]\\
&+\E\left[\sum_{t=j}^{H}r_t\nu_{0:t}g_j \mid s_j,a_{j-1},s_{j-1}\right] \\
&= \E[\nu_{0:j-1}|a_{j-1},s_{j-1}]\E\left[\sum_{t=j}^{H}r_t\nu_{j:t} \left\{\sum_{i=j+1}^{t}g_i\right\}\mid s_j\right]+\E[\nu_{0:j-1}\left\{\sum_{i=0}^{j-1}g_i\right\}|a_{j-1},s_{j-1}] \E\left[\sum_{t=j}^{H}r_t\nu_{j:t} \mid s_j\right] \\
& +\E_{\thpol}[Q_j g_j \mid s_j]. 
\end{align*}
To sum up, the efficient influence function of $Z(\theta)$ under MDP is 
\begin{align}\ts 
\label{eq:eff_mdp_calculated}
      &\sum_{j=0}^{H}\{d^{\mu}_j(s_j,a_j) r_j-\mu_j(s_j,a_j)d^{q}_j(s_j,a_j)-d^{\mu}_j(s_j,a_j)q_j(s_j,a_j) \\
      &\qquad+\mu_{j-1}(s_{j-1},a_{j-1})d^{v}_j(s_j) +d^{\mu}_{j-1}(s_{j-1},a_{j-1})v_j(s_j)\},  \nonumber
\end{align}
where \begin{align*}\ts 
\E[\nu_{0:j} \mid s_j,a_j] &=\mu_j(s_j,a_j),  \\
\E\left[\sum_{t=j+1}^{H}r_t\nu_{j+1:t} \left\{\sum_{i=j+1}^{t}g_i\right\}\mid s_{j},a_{j}\right] & =d^{q}_j(s_j,a_j), \\
\E\left[\nu_{0:j}\left\{\sum_{i=0}^{j}g_i\right\} \mid s_j,a_j\right] &=d^{\mu}_j(s_j,a_j), \\
\E_{\thpol}[d^{q}_j(s_j,a_j)+q_j g_j \mid s_j] &=d^{v}(s_j). 
\end{align*}

\textbf{Part 2: calculating the variance.} 
Next, we calculate the variance of the efficient influence function using a law of total variance:
\begin{align*}\ts 
   & \var \left[d^{\mu}_0(s_0,a_0)+\sum_{j=0}^{H}d^{\mu}_j(s_j,a_j)\{r_j-q_j(s_j,a_j)+v_{j+1}(s_{j+1})\}+\mu_j(s_j,a_j)\{d^{v}_{j+1}(s_{j+1})-d^{q}_j(s_j,a_j)\} \right] \\
    &=\sum_{k=0}^{H+1}\E\left[\var \left[\E[d^{\mu}_0(s_0,a_0)+\sum_{j=0}^{H}d^{\mu}_j(s_j,a_j)\{r_j-q_j(s_j,a_j)+v_{j+1}(s_{j+1})\}+\mu_j(s_j,a_j)\{d^{v}_{j+1}(s_{j+1})-d^{q}_j(s_j,a_j)\}|\ck_{a_k}]|\ck_{a_{k-1}}\right] \right]\\
      &=\sum_{k=0}^{H+1}\E\left[\var \left[\E[\sum_{j=k-1}^{H}d^{\mu}_j(s_j,a_j)\{r_j-q_j(s_j,a_j)+v_{j+1}(s_{j+1})\}+\mu_j(s_j,a_j)\{d^{v}_{j+1}(s_{j+1})-d^{q}_j(s_j,a_j)\}|\ck_{a_k}]|\ck_{a_{k-1}}\right] \right]\\
 &=\sum_{k=0}^{H+1}\E\left[\var \left[d^{\mu}_{k-1}(s_{k-1},a_{k-1})\{r_{k-1}-q_{k-1}(s_{k-1},a_{k-1})+v_{k}(s_{k})\}+\mu_{k-1}(s_{k-1},a_{k-1})\{d^{v}_k(s_{k})-d^{q}_{k-1}(s_{k-1},a_{k-1})\}|\ck_{a_{k-1}}\right] \right] \\ 
  &=\sum_{k=0}^{H+1}\E\left[\var \left[d^{\mu}_{k-1}(s_{k-1},a_{k-1})\{r_{k-1}-q_{k-1}(s_{k-1},a_{k-1})+v_{k}(s_{k})\}+\mu_{k-1}(s_{k-1},a_{k-1})\{d^{v}_k(s_{k})-d^{q}_{k-1}(s_{k-1},a_{k-1})\}|s_{k-1},a_{k-1}\right] \right]. 
\end{align*}
From the third line to the fourth line, we have used the following Bellman equations:
\begin{align*}\ts 
    0 = \E[r_k-q_{k}+v_{k+1}\mid s_k,a_k],\quad0 = \E[-d^{q}_{k}+d^{v}_{k+1}\mid s_k,a_k]. 
\end{align*}
Next, note that
$$d^{\mu}_j(s,a)=\ts \mu_j(s,a)\nabla \log p^{\thpol}_j(s,a).$$
Therefore, the variance is written as 
\begin{align}\ts
\label{eq:variance_mdp}
  &\sum_{k=0}^{H}\E[\mu^2_{k-1}(s_{k-1},a_{k-1})\var [\nabla \log p^{\thpol}_{k-1}(s_{k-1},a_{k-1})\{r_{k-1}-q_{k-1}(s_{k-1},a_{k-1})+v_{k}(s_{k})\} \nonumber\\
  &\,\,\, +\{d^{v}_k(s_{k})-d^{q}_{k-1}(s_{k-1},a_{k-1})\}|s_{k-1},a_{k-1}]].  
\end{align}
\begin{remark}[More specific presentation of the variance]
Note that by covariance formula, the above efficiency bound is equal to
\begin{align*}\ts 
  & \sum_{k=0}^{H+1}\E\left[\mu^2_{k-1}(s_{k-1},a_{k-1})\{\otimes \nabla \log p^{\thpol}_{k-1}(s_{k-1},a_{k-1})\}\var \left[r_{k-1}|s_{k-1},a_{k-1}\right] \right]\\
 &+\sum_{k=0}^{H+1}\E\left[\mu^2_{k-1}(s_{k-1},a_{k-1})\nabla \log p^{\thpol}_{k-1}(s_{k-1},a_{k-1}) \E\left[\{r_{k-1}-q_{k-1}(s_{k-1},a_{k-1})+v_{k}(s_{k})\}d^{v}_k(s_{k})^{\top}|s_{k-1},a_{k-1}\right] \right]\\
 &+\sum_{k=0}^{H+1}\E\left[\mu^2_{k-1}(s_{k-1},a_{k-1})\var \left[d^{v}_k(s_{k})|s_{k-1},a_{k-1}\right] \right]. 
\end{align*}
\end{remark}

\textbf{Part 3: a simple bound for the variance.} 

Consider the on-policy case when $\mu_t=1$. Then, from \eqref{eq:variance_mdp}, the efficiency bound of $Z(\theta)$ under MDP is 
\begin{align}\ts
\label{eq:on_policy}
\sum_{k=0}^{H+1}\E_{p^{\thpol}}\left[\var \left[\nabla \log p^{\thpol}_{k-1}(s_{k-1},a_{k-1})\{r_{k-1}-q_{k-1}(s_{k-1},a_{k-1})+v_{k}(s_{k})\}+\{d^{v}_k(s_{k})-d^{q}_{k-1}(s_{k-1},a_{k-1})\}|s_{k-1},a_{k-1}\right] \right]. 
\end{align}
Since this is the lower bound regarding asymptotic MSE among regular estimators of $Z(\theta)$, it is smaller than the variance of 
\begin{align*}\ts
\sum_{t=0}^{H}r_t\sum_{k=0}^{t}g_k(s_k\mid a_k), 
\end{align*}
noting $\E_n[\sum_{t=0}^{H}r_t\sum_{k=0}^{t}g_k(s_k\mid a_k)]$ is an asymptotic linear estimator. 
The variance of this estimator is bounded by
\begin{align}\ts
\label{eq:spe}
  \var_{p^{\thpol}}\left[\sum_{t=0}^{H}r_t\sum_{k=0}^{t}g_k(s_k\mid a_k)\right] & \preceq \nonumber R^2_{\max}\var_{p^{\thpol}}\left[\sum_{t=0}^{H}\sum_{k=0}^{t}g_k(s_k\mid a_k)\right]  \nonumber\\ 
  &=R^2_{\max}\sum_{t=0}^{H}\sum_{k=0}^{t} \var_{p^{\thpol}}\left[g_k(s_k\mid a_k)\right] \nonumber\\
  &\preceq R^2_{\max}G_{\max}\left\{\frac{(H+1)(H+2)}{2}\right\}^2\mathcal{I}_{D\times D}. 
\end{align}
Here, from the first line to the second line, we use the fact that the covariance across the time is zero:
\begin{align*}
    \cov_{p^{\thpol}}[g_k(s_k\mid a_k),g_j(s_j\mid a_j)]= 0,\,(k\neq j), 
\end{align*}
since when $k< j$
\begin{align*}
      \cov_{p^{\thpol}}[g_k(s_k\mid a_k),g_j(s_j\mid a_j)]=\E_{p^{\thpol}}[g_k g_j]-\E_{p^{\thpol}}[g_k ]\E_{p^{\thpol}}[g_j ]=\E_{p^{\thpol}}[g_k\E_{p^{\thpol}}[g_j\mid s_j,a_j]]=0. 
\end{align*}
Therefore, the quantity \eqref{eq:on_policy} is also bounded by RHS of \eqref{eq:spe}. 

Let us go back to the general off--policy case. For any functions $k(s_t,a_t)$ taking a real number, by importance sampling, we have 
\begin{align*}\ts
    \E_{p^{\bpol}}[\mu^2_t(s_t,a_t)k(s_t,a_t)]=\E_{p^{\thpol}}[\mu_t(s_t,a_t)k(s_t,a_t)] \leq  \E_{p^{\thpol}}[k(s_t,a_t)]C_2, 
\end{align*}
since $\mu_t$ is upper bounded by $C_2$. Therefore, noting the difference of \eqref{eq:variance_mdp} and  \eqref{eq:on_policy}, the quantity $\|\var[\xi_{\MDP}]\|_{\oper}$ is upper-bounded by 
\begin{align*}\ts
    C_2R^2_{\max}G_{\max}\left\{\frac{(H+1)(H+2)}{2}\right\}^2. 
\end{align*}
\end{proof}

\begin{proof}[Proof of Theorem  \ref{thm:nmdp}]
We omit the proof of the first and second parts since it is almost the same as Theorem  \ref{thm:nmdp}, where we simply replace $\mu_t(s_t,a_t),q_t(s_t,a_t),d^{q}_t(s_t,a_t),d^{\mu}_t(s_t,a_t)$ with $\nu_{0:t}(\ch_{a_t}),q_t(\ch_{a_t}),d^{q}_t(\ch_{a_t}),d^{\nu}_t(\ch_{a_t})$. Then, based on \eqref{eq:eff_mdp_calculated}, the efficient influence function of $Z(\theta)$ under NMDP is 
\begin{align*}\ts 
     \xi_{\NMDP}=&\sum_{j=0}^{H}\{d^{\nu}_j(\ch_{a_j}) r_j-\nu_{0:j}(\ch_{a_j})d^{q}_j(\ch_{a_j})-d^{\nu}_j(\ch_{a_j})q_j(\ch_{a_j}) \\
      &+\nu_{0:j-1}(\ch_{a_{j-1}})d^{v}_j(\ch_{s_j}) +d^{\nu}_{j-1}(\ch_{a_{j-1}})v(\ch_{s_j})\}. 
\end{align*}
The efficiency bound of $Z(\theta)$ under NMDP is 
\begin{align}\ts
\label{eq:variance_nmdp}
  \sum_{k=0}^{H+1}\E[\nu^2_{k-1}(\ch_{a_{k-1}})\var [\nabla \log p^{\thpol}_{k-1}(\ch_{a_{k-1}})\{r_{k-1}-q_{k-1}(\ch_{a_{k-1}})+v_{k}(\ch_{s_{k}})\}  +\{d^{v}_k(\ch_{s_{k}})-d^{q}_{k-1}(\ch_{a_{k-1}})\}|\ch_{a_{k-1}}]],
\end{align}
where $\nabla \log p^{\thpol}_{k}(\ch_{a_{k}})=\sum_{j=0}^{k} g_j(\ch_{a_j})$. 
Again, consider the on-policy case where $\nu_{0:t}=1$. Then, the above is equal to 
\begin{align}\ts
  \sum_{k=0}^{H+1}\E_{p^{\thpol}}[\var [\nabla \log p^{\thpol}_{k-1}(\ch_{a_{k-1}})\{r_{k-1}-q_{k-1}(\ch_{a_{k-1}})+v_{k}(\ch_{s_{k}})\}  +\{d^{v}_k(\ch_{s_{k}})-d^{q}_{k-1}(\ch_{a_{k-1}})\}|\ch_{a_{k- 1}}]]. 
\end{align}
Again, this quantity is bounded by RHS of \eqref{eq:spe}. Go back to the general off-policy case. For any functions $k(s_t,a_t)$ taking a real number, by importance sampling, we have 
\begin{align*}\ts
    \E_{p^{\bpol}}[\nu^2_t(s_t,a_t)k(s_t,a_t)]=\E_{p^{\thpol}}[\nu_{0:t}(s_t,a_t)k(s_t,a_t)] \leq  \E_{p^{\thpol}}[k(s_t,a_t)]C^{t}_1, 
\end{align*}
noting $\nu_{0:t} \leq C^t_1$. Therefore, noting the difference of \eqref{eq:variance_mdp} and  \eqref{eq:on_policy}, the term $\|\var[\xi_{\MDP}]\|_{\oper}$ is upper-bounded by 
\begin{align*}\ts
    C^{H}_1R^2_{\max}G_{\max}\left\{\frac{(H+1)(H+2)}{2}\right\}^2.  
\end{align*}
\end{proof}

\begin{proof}[Proof of Theorem \ref{thm:lower_bound_nmdp}]

The efficiency bound of $Z(\theta)$ under NMDP is written as $
\sum_{k=0}^{H+1}\E\left[\nu^2_{k-1} \alpha_{k-1}(\ch_{k-1}) \right]$, 
where 
\begin{align*}
    \alpha_{k-1}(\ch_{k-1})=\var [\nabla \log p^{\thpol}_{k-1}(\ch_{a_{k-1}})\{r_{k-1}-q_{k-1}(\ch_{a_{k-1}})+v_{k}(\ch_{s_{k}})\}  +\{d^{v}_k(\ch_{s_{k}})-d^{q}_{k-1}(\ch_{a_{k-1}})\}|\ch_{a_{k-1}}]. 
\end{align*}
From the assumption, this efficiency bound is lower bounded:
\begin{align*}\ts
\sum_{k=0}^{H+1}\E_{p^{\bpol}}\left[\nu^2_{k-1} \alpha_{k-1}(\ch_{a_{k-1}}) \right] & \succeq \E_{p^{\bpol}}\left[\nu^2_{H} \alpha_H(\ch_{a_{H}} ) \right]\succeq C^{2H}_3\E_{p^{\bpol}}\left[\alpha_H(\ch_{a_{H}})\right]\succeq C^{2H}_3c. 
\end{align*}
Here, we also have used $\alpha_{k}(\ch_{a_k})$ for each $-1 \leq k \leq H-1$ is a semi-positive definite matrix, and 
\begin{align*}
     \alpha_{H}(\ch_{a_{H}})= \var[\nabla \log p^{\thpol}_{H}(\ch_{a_{H}})\{ r_H-q_H\} \mid \ch_{a_H}]. 
\end{align*}
\end{proof}

\begin{proof}[Proof of Theorem \ref{thm:step_wise}]
We have 
\begin{align*} \ts
\var[\sum_{t=0}^{H+1}\nu_{0:t}r_t\sum_{s=0}^t g_s] &=\sum_{k=0}^{H+1}\E[\var[\E[ \sum_{t=0}^{H}\nu_{0:t}r_t\sum_{s=0}^t g_s  \mid \ck_{a_k}]\mid \ck_{a_{k-1}}]]  \\  
 &=\sum_{k=0}^{H+1}\E[\var[\E[ \sum_{t=k-1}^H\nu_{0:t}r_t\sum_{s=k-1}^{t} g_s  \mid \ck_{a_k}]\mid \ck_{a_{k-1}}]]   \\ 
 &=\sum_{k=0}^{H+1}\E[\nu^2_{k-1}\var[\E[\sum_{t=k-1}^H\nu_{k:t}r_t\sum_{s=k-1}^{t} g_s  \mid \ck_{a_k}]\mid \ck_{a_{k-1}}]]. \\
 &=\sum_{k=0}^{H+1}\E[\nu^2_{k-1}\var[\E[\sum_{t=k-1}^H\nu_{k:t}r_t\sum_{s=k-1}^{t} g_s  \mid \ch_{a_k}]\mid \ch_{a_{k-1}}]]. 
\end{align*}
\end{proof}

\begin{proof}[Proof of Theorem \ref{thm:step_wise_lower}]
Based on \cref{thm:step_wise}, as in the proof of \cref{thm:lower_bound_nmdp}, when $c\mathcal{I}\preceq \var[r_Hg_H\mid \ch_{a_H}]$, this variance is lower bounded by $C^{2H}_3c$. 
\end{proof}

\begin{proof}[Proof of Theorem  \ref{thm:db}]

For the simplicity of the notation, we prove the case where $K=2$. Recall that the influence function of $\xi_{\MDP}$ is 
\begin{align}\ts 
      \xi_{\MDP}(\ck;q,\mu,d^{q},d^{\mu})&=\sum_{j=0}^{H}\{d^{\mu}_j(s_j,a_j) r_j-\mu_j(s_j,a_j)d^{q}_j(s_j,a_j)-d^{\mu}_j(s_j,a_j)q_j(s_j,a_j) \\
      &+\mu_{j-1}(s_{j-1},a_{j-1})d^{v}_j(s_j) +d^{\mu}_{j-1}(s_{j-1},a_{j-1})v_j(s_j)\}. 
\end{align}
Here, $\mu=\{\mu_j\},q=\{q_j\},\,d^{q}=\{d^{q}_j\},\,d^{\mu}=\{d^{\mu}_j\}$. 
Then, the estimator $\hat Z^{\DO}(\theta)$ is 
\begin{align*}\ts
    0.5\E_{\cu_1}[\xi_{\MDP}(\ck;\hat q^{(1)},\hat \mu^{(1)},\hat d^{q(1)},\hat d^{\mu(1)})]+0.5\E_{\cu_2}[\xi_{\MDP}(\ck;\hat q^{(2)},\hat \mu^{(2)},\hat d^{q(2)},\hat d^{\mu(2)})]. 
\end{align*}
Then, we have 
\begin{align}\ts
    &\sqrt{n}\{\E_{\cu_2}[\xi_{\MDP}(\ck;\hat q^{(2)},\hat \mu^{(2)},\hat d^{q(2)},\hat d^{\mu(2)})]-Z(\theta)\}  \nonumber \\ 
    &=\sqrt{n}\{\bG_{\cu_2}[ \xi_{\MDP}(\ck;\hat q^{(2)},\hat \mu^{(2)},\hat d^{q(2)},\hat d^{\mu(2)})-\xi_{\MDP}(\ck;q,\mu,d^q,d^\mu) ] \label{eq:first} \\
    &+\sqrt{n}\{\E[\xi_{\MDP}(\ck;\hat q^{(2)},\hat \mu^{(2)},\hat d^{q(2)},\hat d^{\mu(2)})\mid \hat q^{(2)},\hat \mu^{(2)},\hat d^{q(2)},\hat d^{\mu(2)}]-\E[\xi_{\MDP}(\ck;q,\mu,d^q,d^\mu)   ] \}\label{eq:second} \\
    &+\sqrt{n}\{\E_{\cu_2}[ \xi_{\MDP}(\ck;q,\mu,d^q,d^\mu) ] -Z(\theta)\}.  \nonumber
\end{align}
The first term \eqref{eq:first} is $\op(1)$ following the proof of Theorem 5 \citep{KallusUehara2019} (Also from doubly robust struture of EIF from the following lemma). The second term \eqref{eq:second} is following Lemma \ref{lem:second}. 

\begin{lemma}
\label{lem:second}
 The term $\eqref{eq:second}$ is $\op(1)$. 
\end{lemma}
\begin{proof}
\begin{align*}\ts
&\E[\xi_{\MDP}(\ck;\hat q^{(2)},\hat \mu^{(2)},\hat d^{q(2)},\hat d^{\mu(2)})\mid \hat q^{(2)},\hat \mu^{(2)},\hat d^{q(2)},\hat d^{\mu(2)}]-\E[\xi_{\MDP}(\ck;q,\mu,d^q,d^\mu)]  \\
&=\E[\sum_{k=0}^{H}(\hat \mu^{(2)}_k-\mu_k)(-\hat d^{q(2)}_k+d^q_k)+(\hat d^{\mu(2)}_{k}-d^{\mu}_{k})(-\hat q^{(2)}_k+q_k)\mid \mathcal{L}_2] \\
&+\E[\sum_{k=0}^{H} (\hat \mu^{(2)}_{k-1}-\mu_{k-1}) (\hat d^{v(2)}_k-d^v_k)+(\hat d^{\mu(2)}_{k-1}-d^{\mu}_{k-1})(\hat v^{(2)}_k-v_k)\mid \mathcal{L}_2  ]   \\ 
&+\E[\sum_{k=0}^{H}d^{\mu}_k(q_k-\hat q^{(2)}_k)+\mu_k(d^{q}_k-\hat d^{q(2)}_k)\mid \mathcal{L}_2] \\
&+\E[\sum_{k=0}^{H}d^{\mu}_{k-1}(\hat v^{(2)}_k-v_k)+\mu_{k-1}(\E_{\epol}[\hat d^{q(2)}_k+\hat q^{(2)}_kg_k\mid s_k ]-\E_{\epol}[d^{q}_k+q_kg_k\mid s_k ]) \mid \mathcal{L}_2] \\
&+\E[\sum_{k=0}^{H}(\hat \mu^{(2)}_k -\mu_k)(- d^{q}_k+d^{v}_{k+1})+(\hat d^{\mu(2)}_k-d^{\mu}_k)(r_k- q_k+v_{k+1}) \mid \mathcal{L}_2]  \\
&=\E[\sum_{k=0}^{H}(\hat \mu^{(2)}_k-\mu_k)(-\hat d^{q(2)}_k+d^q_k)+(\hat d^{\mu(2)}_{k}-d^{\mu}_{k})(-\hat q^{(2)}_k+q_k)\mid \mathcal{L}_2] \\
&+\E[\sum_{k=0}^{H} (\hat \mu^{(2)}_{k-1}-\mu_{k-1}) (\hat d^{v(2)}_k-d^v_k)+(\hat d^{\mu(2)}_{k-1}-d^{\mu}_{k-1})(\hat v^{(2)}_k-v_k)\mid \mathcal{L}_2  ] .
\end{align*}
Here, we use 
\begin{align*}\ts 
0&= \E[\mu_k f(s_k,a_k)-\mu_{k-1}\E_{\thpol}[f(s_k,a_k)\mid s_k ]], \\ 
0&= \E[d^{\mu}_k f(s_k,a_k)-d^{\mu}_{k-1}\E_{\thpol}[f(s_k,a_k)\mid s_k ]-\mu^{k-1}\E_{\thpol}[f(s_k,a_k)g_k \mid s_k ]],\\ 
    0 &= \E[f(s_k,a_k)(r_k-q_{k}+v_{k+1})],  \\ 
   0 &= \E[f(s_k,a_k)(-d^{q}_{k}+d^{v}_{k+1})]. 
\end{align*}
Then, from \Holder's inequality, the Euclidean norm of the above is upper bounded by up to some absolute constant: 
\begin{align*}\ts
&\sum_{k=0}^{H}\|\hat \mu^{(2)}_k-\mu_k \|_{L^2_b}\|\|\hat d^{q(2)}_k-d^{q}_k \|_{2}\|_{L^2_b}+\|\|\hat d^{\mu(2)}_k-d^{\mu}_k\|_{2}\|_{L^2_b}\|\hat q^{(2)}_k-q_k \|_{L^2_b} \\ 
&+\sum_{k=0}^{H} \|\hat \mu^{(2)}_{k-1}-\mu_{k-1} \|_{L^2_b}\|\|\hat d^{v(2)}_k-d^{v}_k \|_{2}\|_{L^2_b}+\|\|\hat d^{\mu(2)}_{k-1}-d^{\mu}_{k-1}\|_{2}\|_{L^2_b}\|\hat v^{(2)}_k-v_k \|_{L^2_b}\\
&=\op(n^{-\alpha_1})\op(n^{-\alpha_4})+\op(n^{-\alpha_2})\op(n^{-\alpha_3})+\op(n^{-\alpha_1})\op(n^{-\min(\alpha_3,\alpha_4)})+\op(n^{-\alpha_2})\op(n^{-\alpha_3})\\
&=\op(n^{-\min\{\alpha_1,\alpha_2\}})\op(n^{-\min\{\alpha_3,\alpha_4\}})=\op(n^{-1/2}). 
\end{align*}
Here, the convergence rates of $\hat v,\,\hat d^{v}$ are proved as follows:
\begin{align*}\ts
    \|\hat v_k-v_k\|^2_{L^2_b} &=\E_{\bpol}[\{\E_{\thpol}[\hat q_k(s_k,a_k)\mid s_k]-\E_{\thpol}[q_k(s_k,a_k)\mid s_k]\}^2\mid \hat q_k] \\ 
    &\leq \E_{\bpol}[\E_{\thpol}[\{\hat q_k(s_k,a_k)-q_k(s_k,a_k)\}^2\mid s_k]\mid \hat q_k] \\ 
    & \leq C_1\E_{\bpol}[\E_{\bpol}[\{\hat q_k(s_k,a_k)-q_k(s_k,a_k)\}^2\mid s_k]\mid \hat q_k] \\ 
      & \leq C_1\E_{\bpol}[\{\hat q_k(s_k,a_k)-q_k(s_k,a_k)\}^2 \mid \hat q_k]=\op(n^{-\alpha_3})\\ 
\end{align*}
The first line to the second line is proved by conditional Jensen's inequality. In the same way, by defining $q_{k,i}$ as a $i$-th component of $q_{k}$,  
\begin{align*}\ts
      \|\hat d^{v}_{k,i}-d^{v}_{k,i}\|^2_{L^2_b} &< \E_{\bpol}[\{\E_{\thpol}[(\hat d^{q}_{k,i}-d^{q}_{k,i})+(\hat q_k-q_k)g_{k,i} \mid s_k]\}^2  \mid \hat q_k,\hat d^{q}_k] \\
      &\leq \E_{\bpol}[\E_{\thpol}[\{(\hat d^{q}_{k,i}-d^{q}_{k,i})+(\hat q_k-q_k)g_{k,i}\}^2 \mid s_k]  \mid \hat q_k,\hat d^{q}_k] \\
      &\leq C_1 \E_{\bpol}[\E_{\bpol}[\{(\hat d^{q}_{k,i}-d^{q}_{k,i})+(\hat q_k-q_k)g_{k,i}\}^2 \mid s_k]  \mid \hat q_k,\hat d^{q}_k] \\
    &\leq C_1 \E_{\bpol}[\{(\hat d^{q}_{k,i}-d^{q}_{k,i})+(\hat q_k-q_k)g_{k,i}\}^2  \mid \hat q_k,\hat d^{q}_k]= \op( n^{-\min(\alpha_3,\alpha_4)}).
\end{align*}
\end{proof}
Finally, combining everything, we have 
\begin{align*}\ts
    & 0.5\E_{\cu_1}[\xi_{\MDP}(\ck;\hat q^{(1)},\hat \mu^{(1)},\hat d^{q(1)},\hat d^{\mu(1)})]+0.5\E_{\cu_2}[\xi_{\MDP}(\ck;\hat q^{(2)},\hat \mu^{(2)},\hat d^{q(2)},\hat d^{\mu(2)})] \\ 
    &=0.5 \E_{\cu_1}[ \xi_{\MDP}(\ck;q,\mu,d^q,d^\mu) ] +0.5 \E_{\cu_2}[ \xi_{\MDP}(\ck;q,\mu,d^q,d^\mu) ] +\op(n^{-1/2}) \\ 
    &= \E_{n}[ \xi_{\MDP}(\ck;q,\mu,d^q,d^\mu) ]+\op(n^{-1/2}) . 
\end{align*}
Finally, CLT concludes the proof. 
\end{proof}

\begin{proof}[Proof of Theorem  \ref{thm:db_weak}]

For the simplicity of the notation, we prove the case where $K=2$. Recall that the influence function of $\xi_{\MDP}$ is 
\begin{align}\ts 
      &\xi_{\MDP}(\ck;q,\mu,d^{q},d^{\mu})=\sum_{j=0}^{H}\{d^{\mu}_j(s_j,a_j) r_j-\mu_j(s_j,a_j)d^{q}_j(s_j,a_j)-d^{\mu}_j(s_j,a_j)q_j(s_j,a_j) \\
      &+\mu_{j-1}(s_{j-1},a_{j-1})\E[d^{q}_j(s_j,a_j)|s_j] +d^{\mu}_{j-1}(s_{j-1},a_{j-1})\E[q_j(s_j,a_j)|s_j]\}. 
\end{align}
Here, $\mu=\{\mu_j\},q=\{q_j\},\,d^{q}=\{d^{q}_j\},\,d^{\mu}=\{d^{\mu}_j\}$. 
Then, the estimator $\hat Z^{\DO}(\theta)$ is 
\begin{align*}\ts
    0.5\E_{\cu_1}[\xi_{\MDP}(\ck;\hat q^{(1)},\hat \mu^{(1)},\hat d^{q(1)},\hat d^{\mu(1)})]+0.5\E_{\cu_2}[\xi_{\MDP}(\ck;\hat q^{(2)},\hat \mu^{(2)},\hat d^{q(2)},\hat d^{\mu(2)})]. 
\end{align*}
Then, we have 
\begin{align}\ts
    &\{\E_{\cu_2}[\xi_{\MDP}(\ck;\hat q^{(2)},\hat \mu^{(2)},\hat d^{q(2)},\hat d^{\mu(2)})]-Z(\theta)\}  \nonumber \\ 
    &\{\bG_{\cu_2}[ \xi_{\MDP}(\ck;\hat q^{(2)},\hat \mu^{(2)},\hat d^{q(2)},\hat d^{\mu(2)})-\xi_{\MDP}(\ck;q^{\dagger},\mu^{\dagger},d^{\mu\dagger},d^{q\dagger}) ] \label{eq:first_} \\
    &+\{\E[\xi_{\MDP}(\ck;\hat q^{(2)},\hat \mu^{(2)},\hat d^{q(2)},\hat d^{\mu(2)})\mid \hat q^{(2)},\hat \mu^{(2)},\hat d^{q(2)},\hat d^{\mu(2)}]-\E[\xi_{\MDP}(\ck;q^{\dagger},\mu^{\dagger},d^{\mu\dagger},d^{q\dagger})   ] \}\label{eq:second_} \\
    &+\{\E_{\cu_2}[ \xi_{\MDP}(\ck;q^{\dagger},\mu^{\dagger},d^{\mu\dagger},d^{q\dagger}) ] -Z(\theta)\}.  \nonumber
\end{align}
The first term \eqref{eq:first_} is $\op(1/\sqrt{n})$ following the proof of Theorem 5 \citep{KallusUehara2019}. The second term \eqref{eq:second_} is $0$ following Lemma \ref{lem:second}.

Finally, 
\begin{align*}\ts
    & 0.5\E_{\cu_1}[\xi_{\MDP}(\ck;\hat q^{(1)},\hat \mu^{(1)},\hat d^{q(1)},\hat d^{\mu(1)})]+0.5\E_{\cu_2}[\xi_{\MDP}(\ck;\hat q^{(2)},\hat \mu^{(2)},\hat d^{q(2)},\hat d^{\mu(2)})] \\ 
    &=0.5 \E_{\cu_1}[ \xi_{\MDP}(\ck;q^{\dagger},\mu^{\dagger},d^{\mu\dagger},d^{q\dagger}) ] +0.5 \E_{\cu_2}[ \xi_{\MDP}(\ck;q^{\dagger},\mu^{\dagger},d^{\mu\dagger},d^{q\dagger}) ] +\op(1) \\ 
    &= \E_{n}[ \xi_{\MDP}(\ck;q^{\dagger},\mu^{\dagger},d^{\mu\dagger},d^{q\dagger}) ]+\op(1) . 
\end{align*}
Finally, the law of large number concludes the proof since the mean is $Z(\theta)$ under the condition in the theorem. We use  $\E[\xi_{\MDP}(\ck;q^{\dagger},\mu^{\dagger},d^{\mu\dagger},d^{q\dagger})]=Z(\theta)$. 

\begin{lemma}
$\E[\xi_{\MDP}(\ck;q^{\dagger},\mu^{\dagger},d^{\mu\dagger},d^{q\dagger})]=Z(\theta).$
\end{lemma}
\begin{proof}
\begin{align*}\ts
&\E[\xi_{\MDP}(\ck;q^{\dagger},\mu^{\dagger},d^{\mu\dagger},d^{\mu \dagger})\mid q^{\dagger},\mu^{\dagger},d^{\mu\dagger}, d^{\mu \dagger}]-\E[\xi_{\MDP}(\ck;q,\mu,d^q,d^\mu)]  \\
&=\E[\sum_{k=0}^{H}(\mu^{\dagger}_k-\mu_k)(-d^{q\dagger}_k+d^q_k)+(d^{\mu \dagger}_{k}-d^{\mu}_{k})(-q^{\dagger}_k+q_k)] \\
&+\E[\sum_{k=0}^{H} (\mu^{\dagger}_{k-1}-\mu_{k-1}) (d^{v\dagger}_k-d^v_k)+( d^{\mu\dagger}_{k-1}-d^{\mu}_{k-1})(v^{\dagger}_k-v_k)  ]   \\ 
&+\E[\sum_{k=0}^{H}d^{\mu}_k(q_k-q^{\dagger}_k)+\mu_k(d^{q}_k-d^{q\dagger}_k) ]\\
&+\E[\sum_{k=0}^{H}d^{\mu}_{k-1}(v^{\dagger}_k-v_k)+\mu_{k-1}(\E_{\thpol}[(q^{\dagger}_k-q_k)g_k\mid s_k] )+\mu_{k-1}(\E_{\thpol}[d^{q\dagger}_k\mid s_k]-\E_{\thpol}[d^{q}_k\mid s_k] ) ] \\
&+\E[\sum_{k=0}^{H}(\mu^{\dagger}_k-\mu_k)(- d^{q}_k+d^{v}_k)+ (d^{\mu \dagger}_k-d^{\mu}_k)(r_k-q_k+v_{k+1})]\\ 
&=\E[\sum_{k=0}^{H}(\mu^{\dagger}_k-\mu_k)(-d^{q\dagger}_k+d^q_k)+(d^{\mu \dagger}_{k}-d^{\mu}_{k})(q^{\dagger}_k-q_k)] \\
&+\E[\sum_{k=0}^{H} (\mu^{\dagger}_{k-1}-\mu_{k-1}) (- d^{v\dagger}_k+d^v_{k+1})+( d^{\mu\dagger}_{k-1}-d^{\mu}_{k-1})(v^{\dagger}_k-v_k)  ]. 
\end{align*}
Here, we use the relations for $\forall f(s,a)$: 
\begin{align*}\ts 
0&= \E[\mu_k f(s_k,a_k)-\mu_{k-1}\E_{\thpol}[f(s_k,a_k)\mid s_k ]], \\ 
0&= \E[d^{\mu}_k f(s_k,a_k)-d^{\mu}_{k-1}\E_{\thpol}[f(s_k,a_k)\mid s_k]-\mu^{k-1}\E_{\thpol}[f(s_k,a_k)g_k \mid s_k ] ],\\  
    0 &= \E[f(s_k,a_k)(r_k-q_{k}+v_{k+1})],  \\ 
   0 &= \E[f(s_k,a_k)(-d^{q}_{k}+d^{v}_{k+1})]. 
\end{align*}
Then, when $\mu=\mu^{\dagger},d^{\mu}=d^{\mu\dagger}$ or $q=q^{\dagger},\,d^{q}=d^{q\dagger}$ or $\mu=\mu^{\dagger},\,q=q^{\dagger}$, we have 
\begin{align*}\ts
&\E[\sum_{k=0}^{H}(\mu^{\dagger}_k-\mu_k)(d^q_k-d^{q\dagger}_k)+(d^{\mu \dagger}_{k}-d^{\mu}_{k})(q_k-q^{\dagger}_k)]+\E[\sum_{k=0}^{H} (\mu^{\dagger}_{k-1}-\mu_{k-1}) (d^{v\dagger}_k-d^v_k)+( d^{\mu\dagger}_{k-1}-d^{\mu}_{k-1})(v^{\dagger}_k-v_k)  ] \\
&=0+0+0+0=0. 
\end{align*}
This concludes the proof. 
\begin{remark}
The above is not equal to $0$ when  $d^\mu=d^{\mu\dagger},\,d^q=d^{q\dagger}$. The reason is in that case: 
\begin{align*}\ts
    \E[\sum_{k=0}^{H} (\mu^{\dagger}_{k-1}-\mu_{k-1}) (d^{v\dagger}_k-d^v_k)]\neq 0. 
\end{align*}
since $d^{v\dagger}_k\neq d^{v}_k$. 
\end{remark}
\end{proof}

\end{proof}

\begin{proof}[Proof of Theorem \ref{thm:monte_carlo}]
First, we have 
\begin{align*}\ts
    q_j(s_j,s_j)=\E[\sum_{t=j}^{H} r_t \nu_{j+1:t}\mid a_j,s_j]. 
\end{align*}
By differentiating w.r.t $\theta$, we have 
\begin{align*}\ts
d^{q}_j(s_j,a_j) &= \E\left[\sum_{t=j}^{H}r_t\nu_{j+1:t} \left\{\sum_{i=j+1}^{t}g_i(a_i|s_i)\right\}\mid a_{j},s_{j}\right]= \E\left[\sum_{t=j+1}^{H}r_t\nu_{j+1:t} \left\{\sum_{i=j+1}^{t}g_i(a_i|s_i)\right\}\mid a_{j},s_{j}\right], 
\end{align*}
noting $$\ts
 \nabla \nu_{j+1:t}=\nu_{j+1:t} \nabla \log \nu_{j+1:t}=\nu_{j+1:t} \{\sum_{i=j+1}^{t}   g_i(a_i|s_i)\}. $$
Second, we have 
\begin{align*}\ts
    \mu_j(s_j,a_j)= \E\left[ \nu_{0:j} \mid a_{j},s_{j}\right]. 
\end{align*}
By differentiating w.r.t $\theta$, we have 
\begin{align*}\ts
   d^{\mu}_j(a_j,s_j)= \E\left[\nu_{0:j}\left\{\sum_{i=0}^{j}g_i(a_i|s_i)\right\} \mid a_j,s_j\right], 
\end{align*}
noting  $$\ts
 \nabla \nu_{0:j}=\nu_{0:j} \nabla \log \nu_{0:j}=\nu_{0:j} \{\sum_{i=0}^{j}g_i(a_i|s_i)\}. $$
\end{proof}

\begin{proof}[Proof of Theorem \ref{thm:Bellman}]
The following recursive equations (Bellman equations) hold:
\begin{align*}\ts
    q_j(s_j,a_j) &= \E[r+ q_{j+1}(s_{j+1},\thpol)\mid s_j,a_j],\\
    \mu_j(s_j,a_j)&= \E[\mu_{j-1}(s_{j-1},a_{j-1})\tilde\nu_{j}|s_j,a_j]. 
\end{align*}
Then, by differentiating w.r.t $\theta$, we have 
\begin{align*}\ts
    d^{q}_j(s_j,a_j) &=  \E[ \E_{\thpol}[d^{q}_{j+1}(s_{j+1},a_{j+1})+g_{j+1}(s_{j+1},a_{j+1})q_{j+1}(s_{j+1},a_{j+1})\mid s_{j+1}] |s_j,a_j], \\
     d^{\mu}_j(s_j,a_j) &=\E[ d^{\mu}_{j-1}(s_{j-1},a_{j-1})|s_j,a_j]\tilde\nu_{j}+\E[\mu_{j-1}(s_{j-1},a_{j-1})|s_j,a_j]g_j(a_j,s_j)\tilde\nu_{j} \\
     &=\E[ d^{\mu}_{j-1}(s_{j-1},a_{j-1})|s_j,a_j]\tilde\nu_{j}+\mu_j(s_j,a_j)g_j(a_j,s_j). 
\end{align*}
\end{proof}

\begin{proof}[Proof of Theorem \ref{thm:optimization} ]
We modify the proof of Theorem 1 \citep{khamaru18a} so that we can deal with the noise gradient. In this proof, define $f(\theta)=-J(\theta)$. Then, by $M$-smoothness,
\begin{align*}\ts
    f(\theta)\leq f(\theta_k)+\langle \nabla f(\theta_k),x-\theta_k \rangle +\frac{M}{2}\|x-\theta_k\|_2. 
\end{align*}
Then, by replacing $\theta$ with $\theta_{k+1}=\theta_k-\alpha_k\nabla f(\theta_k)-\alpha_kB_k $, 
\begin{align*}\ts
    f(\theta_{k+1})\leq f(\theta_k)+\langle \nabla f(\theta_k),\theta_{k+1}-\theta_k \rangle +\frac{M}{2}\|\theta_{k+1}-\theta_k\|_2. 
\end{align*}
Thus, 
\begin{align*}\ts
    f(\theta_k)-f(\theta_{k+1}) & \geq -\langle \nabla f(\theta_k),\theta_{k+1}-\theta_k \rangle -\frac{M}{2}\|\theta_{k+1}-\theta_k\|^2_2 \\ 
     & = \alpha_k \langle \nabla f(\theta_k),\nabla f(\theta_k)+B_k \rangle -\frac{M\alpha_k^2}{2}\|\nabla f(\theta_k)+B_k\|^2_2 \\ 
    & = \alpha_k \|\nabla f(\theta_k)\|^2_2+\alpha_k \langle \nabla f(\theta_k), B_k\rangle -\frac{M\alpha_k^2}{2}\|\nabla f(\theta_k)+B_k\|^2_2 \\ 
     & = \alpha_k \|\nabla f(\theta_k)\|^2_2-\alpha_k |\langle \nabla f(\theta_k), B_k\rangle| -\frac{M\alpha_k^2}{2}\|\nabla f(\theta_k)+B_k\|^2_2 \\   
  & \geq \alpha_k \|\nabla f(\theta_k)\|^2_2-0.5\alpha_k(\|\nabla f(\theta_k)\|^2+\|B_k\|^2_2) -M\alpha_k^2(\|\nabla f(\theta_k)\|^2_2+\|B_k\|^2_2) \\ 
  & \geq 0.25\alpha_k \|\nabla f(\theta_k)\|^2_2-0.5\alpha_k \|B_k\|^2_2-0.25\alpha_k \|B_k\|^2_2.
\end{align*}
Here, from the fourth line to the fifth line, we use inequalities parallelogram law:
\begin{align*}\ts
    2|\langle a,b \rangle| < \|a\|^2_2+\|b\|^2_2,\,
     \|a+b\|^2_2 \leq 2\|a\|^2_2+2\|b\|^2_2. 
\end{align*}
From the fifth line to the sixth line, we use a condition regarding $M$. This yields, 
\begin{align*}\ts
      f(\theta_k)-f(\theta_{k+1})+0.75\alpha_k\|B_k\|^2_2 \geq 0.25\alpha_k \|\nabla f(\theta_k)\|^2_2.
\end{align*}
Then, by telescoping sum,  
\begin{align*}\ts
   \frac{1}{T}\{f(\theta_1)-f^{*}\}+\frac{1}{T}\sum_t 0.75\alpha_t\|B_t\|^2_2 \geq \frac{1}{T}\sum_t 0.25\alpha_t \|\nabla f(x_t)\|^2_2.
\end{align*}
Noting $f(\theta)=-J(\theta)$, 
\begin{align*}\ts
   \frac{1}{T}\{J^{*}-J(\theta_1)\}+\frac{1}{T}\sum_t 0.75\alpha_t\|B_t\|^2_2 \geq \frac{1}{T}\sum_t 0.25\alpha_t \|\nabla J(x_t)\|^2_2.
\end{align*}
Expanding by $4$ yields the result.
\end{proof}

\begin{proof}[Proof of Theorem \ref{thm:error}]
Here,  
\begin{align*}\ts
    B_t=\E_n[\xi_{\MDP}(\theta_t)-Z(\theta_t)]+\op(n^{-1/2}). 
\end{align*}
from the proof of Theorem \ref{thm:db}. Then, the $j^{\text{th}}$ component of $B^{2}_{t}$ is 
\begin{align}\ts
\label{eq:component}
    B^{2}_{t,j}=(\E_n[\xi_{\MDP,j}(\theta_t)-Z_j(\theta_t)]+\op(n^{-1/2}))^2= \E_n[\xi_{\MDP,j}(\theta_t)-Z_j(\theta_t)]^2+\op(n^{-1}),
\end{align}
where $\xi_{\MDP,j}$ is a $j^{\text{th}}$ component of IF and $Z_{j}$ is a $j$-th term of $Z(\theta)$, $\xi_{\MDP}(\theta)$ is $\xi_{\MDP}$ at $\theta$. Here, we use $O_p(n^{-1/2})\op(n^{-1/2})=\op(n^{-1}),\,\op(n^{-1/2})\op(n^{-1/2})=\op(n^{-1})$. 
Here, noting $\theta_t$ is a random variable, we have to bound the main term uniformly as 
\begin{align*}\ts
   \E_n[\xi_{\MDP,j}(\theta_t)-Z_j(\theta_t)]^2\leq (\sup_{\theta \in \Theta}\E_n[\xi_{\MDP,j}(\theta)-Z_j(\theta)])^2. 
\end{align*}

By following Theorem 8.5 \citep{empirical} based on a standard empirical process theory combining Rademacher complexity and Talagland inequality, with probability $1-\delta$, 
\begin{align*}\ts
    & \sup_{\theta \in \Theta}\E_n[\xi_{\MDP,j}(\theta)-Z_j(\theta)] \\
    &\lnapprox  \E[\sup_{\theta \in \Theta}|\frac{1}{n}\sum_{i=1}^{n} \epsilon^{(i)} \xi^{(i)}_{\MDP,j}(\theta)| ]+\sqrt{\frac{\sup_{\theta \in \Theta} \var[\xi_{\MDP,j}] \log(1/\delta) }{n}}+\frac{\log (1/\delta)}{n} \\
    &\lnapprox L\sqrt{D/n}+ \sqrt{ \frac{C_2 G_{\max} R^2_{\max}(H+1)^2(H+2)^2 \log(1/\delta)}{n}}+\frac{\log (1/\delta)}{n}. 
\end{align*}
Then, with probability $1-\delta$, 
\begin{align*}\ts
    & \{\sup_{\theta \in \Theta}\E_n[\xi_{\MDP,j}(\theta)-Z_j(\theta)]\}^2 \\
    &\lnapprox \left \{L\sqrt{D/n} + \sqrt{ \frac{C_2 G_{\max}R^2_{\max}(H+1)^2(H+2)^2 \log(1/\delta)}{n}}+\frac{\log (1/\delta)}{n}\right\}^2. 
\end{align*}
Here, we also use the Rademacher complexity of the Lipschitz class on Euclidean ball is bounded by $L\sqrt{D/n}$ based on the assumption $\Theta$ is a compact space (Example 4.6 \citep{empirical}). 

Considering an error term $\op(1/n)$ in \eqref{eq:component} and taking an union bound over $t\in [1,\cdots,T],j\in [1,\cdots,D]$, we conclude that there exists $N_{\delta}$ such that with probability at least $1-\delta$, $\forall n\geq N_{\delta}$ 
\begin{align*}\ts
     \frac{1}{T}\sum_t \|B_t\|^2_{2}  &=    \frac{1}{T}\sum_{t=1}^T \sum_{j=1}^D B^2_{t,j}  \\
     &\lnapprox D\left \{L\sqrt{\frac{D}{n}} + \sqrt{ \frac{C_2  G_{\max}R^2_{\max}(H+1)^2(H+2)^2 \log(TD/\delta)}{n}}+\frac{\log (TD/\delta)}{n}\right\}^2 \\ 
         &\lnapprox D \frac{L^2 D+C_2 G_{\max}R^2_{\max}(H+1)^2(H+2)^2 \log(TD/\delta)}{n}. 
\end{align*}
\end{proof}

\begin{proof}[Proof of \cref{thm:convexity}]

We modify the proof of Theorem 3.1 \citep{HazanElad2015ItOC} so that we can deal with a noise gradient. Define $-J(\theta)$ as $f(\theta)$ and redefine $Z(\theta),\hat Z(\theta)$ as $-Z(\theta),-\hat Z^{\DO}(\theta)$. Then, the algorithm is redefined as 
\begin{itemize}
    \item $\tilde \theta_t=\theta_t-\alpha_t \hat Z(\theta_t)$
    \item $\theta_t=\Proj_{\Theta} \tilde \theta_t$
\end{itemize}

Then, from convexity assumption for $-J(\theta)$, 
\begin{align*}\ts
    f(\theta_t)-f(\theta^{*})\leq Z(\theta_t) (\theta_t-\theta^{*}). 
\end{align*}
In addition, from Theorem 2.1 \citep{HazanElad2015ItOC}, 
\begin{align*}\ts
    \|\theta_{t+1}-\theta^{*} \|_{2} = \|\Proj_{\Theta}(\theta_t-\alpha_t \hat Z(\theta_t))-\theta^{*}   \|_{2}\leq \|\theta_t-\alpha_t \hat Z(\theta_t)-\theta^{*} \|_{2}.  
\end{align*}
Hence,
\begin{align}\ts \label{eq:inner_inner}
    2 \langle \hat Z(\theta_t), (\theta_t-\theta^{*})\rangle \leq \frac{\|\theta_t-\theta^{*} \|^2_{2}-\| \theta_{t+1}-\theta^{*}\|^2_{2} }{\alpha_t}+\alpha_t \|\hat Z(\theta_t) \|^2_{2}. 
\end{align}
Noting 
\begin{align*}\ts
    \|\hat Z(\theta_t) \|^2_{2} =\|B_t+Z(\theta_t)\|^2_{2} \leq 2\|B_t\|^2_{2}+ 2\|Z(\theta_t)\|^2_{2},  
\end{align*}
as in the proof of Theorem \ref{thm:error}, there exists $N_{\delta}$ such that $n\geq N_{\delta}$ with probability at least $1-\delta$, 
\begin{align*}\ts
    2 \langle Z(\theta_t), (\theta_t-\theta^{*})\rangle \lnapprox  \frac{\|\theta_t-\theta^{*} \|^2_{2}-\| \theta_{t+1}-\theta^{*}\|^2_{2} }{\alpha_t}+(\sup_{\theta \in \Theta } \|Z(\theta)\|^2+\tilde U)\alpha_t,
\end{align*}
where $\tilde U= D \frac{L^2 D+C_2 G_{\max}R^2_{\max}(H+1)^2(H+2)^2 \log(D/\delta)}{n}.$ 
Then, based on 
\begin{align}\ts
\label{eq:target}
   \sum_{t=1}^{T} f(\theta_t)-f(\theta^{*}) & \leq \sum_{t=1}^{T} \langle Z(\theta),(\theta_t-\theta^{*}) \rangle \nonumber \\ 
   &\leq \sum_{t=1}^{T}\langle \hat Z^{\DO}(\theta), (\theta_t-\theta^{*}) \rangle +\sum_{t=1}^{T} \langle Z(\theta)-\hat Z^{\DO}(\theta), (\theta_t-\theta^{*}) \rangle, 
\end{align}
we analyze the first term and second term of \eqref{eq:target}.

\textbf{First term of \eqref{eq:target}}

From \eqref{eq:inner_inner}, there exists $N_{\delta}$ such that $n> N_{\delta}$ with at least $1-\delta$,
\begin{align*}\ts
    \sum_{t=1}^{T} \langle \hat Z^{\DO}(\theta), (\theta_t-\theta^{*})\rangle \lnapprox   \sum_{t=1}^{T} \frac{\|\theta_t-\theta^{*} \|^2_{2}-\| \theta_{t+1}-\theta^{*}\|^2_{2} }{\alpha_t}+(U+\sup_{\theta \in \Theta}\| Z(\theta)\|_2)\alpha_t. 
\end{align*}
where $U= D \frac{L^2 D+C_2 G_{\max}R^2_{\max}(H+1)^2(H+2)^2 \log(TD/\delta)}{n}.$  (We also take an union bound over $t$). Then, under this event, 
\begin{align*}\ts
    & \sum_{t} \frac{\|\theta_t-\theta^{*} \|^2_{2}-\| \theta_{t+1}-\theta^{*}\|^2_{2} }{\alpha_t}+ \sum_t \{ \sup \|Z(\theta)\|^2_2+U\} \alpha_t \\
    &\leq  \sum_t \|\theta_t-\theta^{*} \|^2_{2}(\frac{1}{\alpha_t}-\frac{1}{\alpha_{t+1}})+ \sum_t \{ \sup \|Z(\theta)\|^2_2+U\} \alpha_t  \\ 
    &\leq  \sum_{t=1}^{T} \Upsilon^2(\frac{1}{\alpha_t}-\frac{1}{\alpha_{t+1}})+ \sum_t \{ \|\sup Z(\theta)\|^2_2+U\} \alpha_t  \\ 
    &\leq  \Upsilon^2\frac{1}{\alpha_T}+\{ \sup Z(\theta)^2+U\}  \sum_t \alpha_t \lnapprox \Upsilon\sqrt{\{ \sup \|Z(\theta)\|^2_2+U\} T},
\end{align*}
Here, we take $\alpha_t=\Upsilon/\sqrt{t \{ \sup_{\theta\in\Theta} \|Z(\theta)\|^2+U\}}$. 
The last inequality follows since $\sum 1/\sqrt{t}\leq  \sqrt{T}$. 

\textbf{Second term  of \eqref{eq:target}}

We have 
\begin{align*}\ts 
\sum_{t=1}^{T} \{Z(\theta)-\hat Z^{\DO}(\theta)\}^{\top}(\theta_t-\theta^{*}) 
\leq \sum_{t=1}^{T} \{Z(\theta)-\hat Z^{\DO}(\theta)\}^{\top}(\theta_t-\theta^{*}) \leq \sum_{t=1}^{T} \|B_t\|_{2} \times  \Upsilon. 
\end{align*} 
Then, as in the proof of Theorem \ref{thm:error}, with probability $1-\delta$, this is bounded by 
\begin{align*}\ts
\sum_{t=1}^{T} \|B_t\|_{2} \times \Upsilon \lnapprox T\times \sqrt{U}\times \Upsilon. 
\end{align*}

\textbf{Combining the first term and second term of \eqref{eq:target}}

We combine the first term and second term of \eqref{eq:target}. Then, we have
\begin{align*}
     f\left( \frac{1}{T}\sum_{t=1}^{T}\theta_t\right)-f(\theta^{*}) &<\frac{1}{T}\left (\sum_{t=1}^{T} f(\theta_t)-f(\theta^{*})\right)\\
     &\lnapprox \sqrt{U}\Upsilon +\Upsilon\sqrt{ \frac{\sup\|Z(\theta)\|^2_2+U }{T}} \\ 
     &\lnapprox \Upsilon \left\{\sqrt{ \frac{\sup\|Z(\theta)\|^2_2 }{T}} + \sqrt{U}\left(1+\frac{1}{\sqrt{T}}\right)\right\}. 
\end{align*}
Here, we use an inequality $\sqrt{x+y}\leq \sqrt{x}+\sqrt{y}$ for $x>0,y>0$. 
\end{proof}

\end{document}